\pdfoutput=1
\documentclass[11pt]{article}
\usepackage[preprint]{acl}
\usepackage{times}
\usepackage{latexsym}
\usepackage[T1]{fontenc}
\usepackage[utf8]{inputenc}
\usepackage{microtype}
\usepackage{inconsolata}
\usepackage{graphicx}
\usepackage{dirtytalk}
\usepackage{placeins}
\usepackage{minitoc}
\usepackage{graphicx} 
\usepackage[utf8]{inputenc} 
\usepackage[T1]{fontenc}    
\usepackage{hyperref}       
\usepackage{url}            
\usepackage{booktabs}       
\usepackage{multirow}
\usepackage{amsfonts}       
\usepackage{nicefrac}       
\usepackage{microtype}      
\usepackage{xcolor}         
\usepackage{minitoc}
\usepackage{algorithm}
\usepackage{algpseudocode}

\usepackage[arrow, matrix, curve]{xy}
\usepackage{calc,multirow,tikz,array,setspace,geometry}

\usepackage{graphicx}
\usepackage{algorithmicx}
\usepackage{algpseudocode}
\usepackage{algorithm}
\usepackage{subcaption}
\usepackage{tikz}
\usetikzlibrary{calc,positioning}
\usetikzlibrary{matrix,calc,positioning,arrows}
\usetikzlibrary{arrows,shapes}
\usetikzlibrary{quotes,arrows.meta}
\usepackage{hyperref}
\usepackage{tcolorbox}
\usepackage{textcomp}
\usepackage{color}
\usepackage{enumitem}
\usepackage{wrapfig}
\usepackage{array}
\usepackage{amsthm}
\usepackage{amssymb}
\usepackage{amsmath}
\usepackage{bbold}

\theoremstyle{plain}
\newtheorem{prop}{Proposition}
\newtheorem{thm}{Theorem}

\newtheorem{example}{Example}

\theoremstyle{definition}
\newtheorem{defn}{Definition}

\theoremstyle{remark}
\newtheorem{rem}{Remark}

\newcommand{\Arond}{\mathcal A}
\newcommand{\Brond}{\mathcal B}

\newcommand{\Hrond}{\mathcal H}
\newcommand{\Irond}{\mathcal I}

\newcommand{\Lrond}{\mathcal L}
\newcommand{\Mrond}{\mathcal M}

\newcommand{\Prond}{\mathcal P}

\newcommand{\Rrond}{\mathcal R}

\newcommand{\vertiii}[1]{{\left\vert\kern-0.25ex\left\vert\kern-0.25ex\left\vert #1 
    \right\vert\kern-0.25ex\right\vert\kern-0.25ex\right\vert}}

\newcommand{\set}[1]{\left\{ #1\right\}}

\renewcommand{\leq}{\leqslant}

\newcommand{\Prob}{\mathbb{P}}

\newcommand{\Yspace}{\mathsf{Y}}
\newcommand{\Xspace}{\mathsf{X}}
\newcommand{\Zspace}{\mathsf{Z}}

\newcommand{\Emb}{Z} 
\newcommand{\EmbSpace}{\mathsf{Z}} 

\newcommand{\TV}{\text{TV}}

\newcommand{\lora}{LoRA}
\newcommand{\imSpace}{\textrm{Im}}
\newcommand{\rk}{\textrm{rank}}
\newcommand{\dist}{\textrm{d}}
\newcommand{\queryProj}{\textrm{Q}}
\newcommand{\valueProj}{\textrm{V}}

\newcommand{\lenincludes}{\mathop{\widetilde{\subset}}}
\usepackage{placeins}
\usepackage{bbm}

\title{Statistical Deficiency for Task Inclusion Estimation}

\begin{document}

\author{
 \textbf{Loïc Fosse\textsuperscript{1,2}},
 \textbf{Frédéric Béchet \textsuperscript{2,4}},
 \textbf{Benoît Favre\textsuperscript{2,3}},
 \textbf{Géraldine Damnati\textsuperscript{1}},
\\
 \textbf{Gwénolé Lecorvé\textsuperscript{1}},
 \textbf{Maxime Darrin\textsuperscript{4,5,6,8}},
 \textbf{Philippe Formont\textsuperscript{4,6,7,8}},
 \textbf{Pablo Piantanida\textsuperscript{4,6,8,9}}
\\
 \textsuperscript{1}Orange Research, Lannion, France,
 \textsuperscript{2}CNRS, LIS, Aix Marseille Université, France,
 \\
 \textsuperscript{3}CNRS, LIG, Univ. Grenoble Alpes, Grenoble, France,
 \\
 \textsuperscript{4}International Laboratory on Learning Systems (ILLS - IRL CNRS), Montréal,
 \\
 \textsuperscript{5}McGill University,
 \textsuperscript{6}Mila - Quebec AI Institute,
 \textsuperscript{7}{\'E}TS Montréal,
 \\
 \textsuperscript{8}Université Paris-Saclay,
 \textsuperscript{9}CNRS, CentraleSupélec.
\\
   {\bf \small Contact}: {\small \tt loic.fosse@orange.com}
}

\tikzstyle{fleche} = [draw, thick, color=black, -latex']
\tikzstyle{dfleche} = [draw, thick, color=black, latex'-latex']
\tikzstyle{flecheDashed} = [draw, dashed, color=black, -latex']
\tikzstyle{DflecheDashed} = [draw, dashed, color=black, latex'-latex']

\maketitle
\begin{abstract}
    Tasks are central in machine learning, as they are the most natural objects to assess the capabilities of current models. The trend is to build general models able to address any task. Even though transfer learning and multitask learning try to leverage the underlying task space, no well-founded tools are available to study its structure. This study proposes a theoretically grounded setup to define the notion of task and to compute the {\bf inclusion} between two tasks from a statistical deficiency point of view. We propose a tractable proxy as information sufficiency to estimate the degree of inclusion between tasks, show its soundness on synthetic data, and use it to reconstruct empirically the classic NLP pipeline.\footnote{Code and data are available \href{https://gitlab.lis-lab.fr/talep-public/acl2025}{here}}
\end{abstract}

\section{Introduction}
\label{sec:intro}

Having a well-defined set of tasks with known or assumed dependency relationships has historically been a key element in building or evaluating Natural Language Processing~(NLP) systems.
For instance, Named Entity Recognition (NER) and summarization are two well-established tasks for which annotated datasets exist, and it is commonly accepted that the summarization task, at least in the news domain, requires NER skills to be performed effectively.
As a consequence, studying generated summaries from the perspective of retained named entities is a relevant evaluation angle~\cite{pagnoni-etal-2021-understanding,berezin2023named,akani2023reducing}.
According to this principle, a more general hypothesis is that multi-task training~\cite{caruana1997multitask} provides cross-task generalization~\cite{naturalinstructions,ye-2024-cross,wang2024generalizationvsmemorizationtracing,baxterModelInductiveBias2000,wuUnderstandingImprovingInformation2020} since tasks share dependencies.
However, with the rise of instruct-tuned models~\cite{wei-arxiv-21}, tasks can be directly defined by prompting large language models.
This dramatically increase of the addressable tasks' space makes the notions of ground truth and labeled dataset more fuzzy, and raises many questions about the capabilities of these seemingly all-powerful models. Notably, it is unclear how many tasks a single model can correctly handle, how many parameters are necessary to capture a given task, and what proportion of the model is dedicated to language understanding, task solving, and memorization.

In this paper, to advance the understanding of the notion of task and its manipulation within language models, we are interested in studying intrinsic relationships between tasks. Building on the intuition that some tasks are necessary conditions to others ({\it e.g.} NER is a necessary condition for summarization), we propose a framework to discover statistical \textit{task inclusion relationships}.
Potential application to the discovery of such relations is to build smaller, more compute-efficient datasets~\cite{zamir-ieee-18} and, more generally, to build better data-mix when training models~\cite{ye2024datamixinglawsoptimizing}, or more orthogonal benchmarks.
Our main approach will rely on statistical simulation methods~\cite{camSufficiencyApproximateSufficiency1964,lecamComparisonExperimentsaShort1996} to decide whether one task can be transformed into another.
While we theoretically show the shortcomings of naively measuring cross-task performance by directly applying each model to each other task, the contributions of the paper are:
\begin{itemize}[leftmargin=*]
    \setlength\itemsep{0em}
    \item {\bf A theoretical framework for task definition and inclusion.} Based on information concepts and theory, we propose a clear definition of a task and candidate notions of inclusion (independent of the notion of model).
    \item {\bf An experimental setup for task comparison.} We propose a tractable proxy to measure task inclusion through statistical reductions.
    \item {\bf An empirical rediscovery of the NLP pipeline} Experiments suggest that our framework reconstructs the expected partial order for a sample of linguistic tasks from the classical NLP pipeline.
\end{itemize}

\section{Related work}
\label{sec:related_work}

Discovering the task space underlying structure~\cite{turing1950computing,winograd1987thinking} is a common problem in machine learning (ML), to compare human and machine representations of tasks but more generally to leverage potential structure for more efficient learning~\cite{li2025optimal}.
We list the main families of task comparison methods and discuss how our new one differs from them.

\paragraph{Task similarity.} Evaluating similarity between tasks is one of the main area of interest, as it is the crux to discovering transfer learning or meta-learning opportunities~\cite{schmidhuber-phd-87, zhou-pmlr-21}. It can be performed by comparing the same model trained on various tasks~\cite{achille-ieee-19,shui-arxiv-19}, by comparing the data distributions of different tasks~\cite{ethayarajh-icml-22}. Unlike work on task similarity, we aim at finding a non-symmetrical notion of task inclusion. We discuss in~\autoref{app:task-vectors} limitations of similarity based approaches and why the proposed setup is more suited for task relationship study.

\paragraph{Task merging.} Inspired from ensemble methods~\cite{dietterich-iwmcs-00}, model merging focuses on combining several existing models to create a new one. Recent methods either use arithmetic operations~\cite{ilharcoEditingModelsTask2023,taoTaskArithmeticLens2024,ortiz-jimenezTaskArithmeticTangent,zhouMetaGPTMergingLarge2024,zeng2025efficient} or more complex aggregation methods~\cite{yadavTIESMergingResolvingInterference2023a,jinDatalessKnowledgeFusion2023,yangAdaMergingAdaptiveModel2023}, with the goal of solving conflicts or interferences between models~\cite{yu2020gradient,sener2018multi} and thus tasks. While task and model merging focuses on the parameter space (whose dimension is excessively large), the tools developed here focus on the activation space. However, we propose in~\autoref{app:task-vectors} an analysis of parameter space based approaches and we show some limitations, despite some interesting behaviors.

\paragraph{Task Transfer.} Transfer learning~\cite{torrey-igi-10,hanneke-arxiv-24,lange-etal-2021-share} consists in leveraging a model pre-trained for a new task for a given task, either as an initialization point for further training or to generate useful representations. Although most of the time one uses a generic pre-trained model and trains directly for the new task, \citet{vu-arxiv-20} showed that some tasks might benefit from training on an intermediate task, effectively building a path of (easily) transferable tasks.
In computer vision, this phenomenon has been studied and quantified by \citet{baoInformationTheoreticApproachTransferability2019}. \citet{zamir-ieee-18}, obtained similar results showing connections between various visual tasks and were able to leverage these structures to optimize training of multitask models~\cite{zhang-ieee-21}. Knowledge transferability between different tasks is also at the heart of modern ML and generalization as exhibited by models such as T5~\cite{wei-arxiv-21,khashabi-arxiv-20}, or more recently instruct-type models~\cite{zhang2023instruction} with multi-task training leading to significantly stronger results. Task transfer is mainly based on successive fine-tuning processes, as well as on the study of the fine-tuned models parameters geometry (Fisher information). The tools developed here focus on fine-tuned model activations enabling us to avoid certain learning costs and connect with powerful theoretical results.

\paragraph{Probing.} Understanding what is encoded in a  model has been a question of interest which led to probing methods. It consists in assessing the activations of a model on a downstream task~\cite{guillaume-iclr-17,chen-arxiv-21,rogers-tacl-21,wallat-ecir-23,nikolaev-emnlp-23,zhao-acm-24,waldisHolmesBenchmarkAssess2024a}. 
Some studies assess encoder-type models with probing methods to understand what the model encodes after fine-tuning on a task~\cite{durraniHowTransferLearning2021,merchantWhatHappensBERT2020,mosbachInterplayFinetuningSentenceLevel2020a}, in order to understand how tasks impact pre-trained models. The main observation is that deepest layers will focus more on the task, while first layers remain generic. This result is confirmed in~\cite{durraniTransformationLatentSpace2022a} with the use of clustering methods. However, probing has some drawbacks as discussed in~\cite{pimentel-arxiv-20,kunz-iccl-20}, where it is clearly explained that it can lead to miss-interpretation. Discussion about probing interpretation is provided in~\autoref{app:information-theory:ce-decomp}. Moreover, the last layer interpretation does not seem to hold for decoder-only generative language models (which we are using here)~\cite{gromov2024unreasonable}, even when fine-tuned on a task ({\it c.f.}~\autoref{sec:results}).

\section{Theoretical framework}
\label{sec:formalism}
\label{formalism}

We introduce here the theoretical set-up we will be using, which is mainly a probabilistic one. After defining a task (\autoref{def:task}), we will start by proposing a strict probabilistic definition of task inclusion, and a relaxed definition (\autoref{def:relaxed-inclusion}) that states that a task is included in another one if the estimation process of the latter provides sufficient information on the former. Then we propose a way of computing this crucial notion of informativeness through statistical deficiency (\autoref{def:deficiency}). Beyond this theoretically grounded approach, we finally propose a tractable way to estimate deficiency through information sufficiency. The approach is developed in the following subsections and can be synthesized in~\autoref{fig:task-comm-comp}.
\paragraph{Notations.} We denote by $\Prond(\Xspace)$ the set of all probability measures on $\Xspace$. For any random variable $X\in\Xspace$, we'll denote by $\Prob_X \in \Prond(\Xspace)$ the associated push-forward measure. Given two probability measures $P$ and $Q$ on the same space, we denote by $\|P-Q\|_{\TV}$, the total variation distance. Given some spaces $\Xspace$ and $\Yspace$, we denote by $\Mrond(\Yspace | \Xspace)$ the set of all Markov Kernels from $\Xspace$ to $\Yspace$, which can be seen (under certain assumptions) as the set of all conditional probability measures $\Prob_{Y|X}$. Given $K \in \Mrond(\Zspace | \Yspace)$ and $M \in \Mrond(\Yspace | \Xspace)$, we denote by $K \circ M \in \Mrond(\Zspace | \Xspace)$ the composition of kernels\footnote{We give further details about the formalism in~\autoref{app:formalism:details}.}.

\subsection{Task and inclusion}\label{subsec:task-inclusion}

\begin{defn}[Task]\label{def:task}
    Given input data $X \in \Xspace$ and response $Y \in \Yspace$, a task is the joint probability measure $\Prob_{XY} \in \Prond(\Xspace\times\Yspace)$. 
\end{defn}
\begin{rem}\label{rem:task-def-notes}
    \autoref{def:task} provides a simple generic definition of tasks in a ML context, which is (at least implicitly) adopted in many works~\cite{maurer2016benefit,baxterModelInductiveBias2000}.  
\end{rem}
Given~\autoref{def:task}, many cases can appear while comparing tasks, yielding different interpretations. Two tasks can be considered on different input marginal distributions $\Prob_X$, which is known in the literature as the {\it domain shift}~\cite{wortsmanRobustFinetuningZeroshot2022,taoriMeasuringRobustnessNatural2020,radfordLearningTransferableVisual2021,kumarFineTuningCanDistort2022}. Tasks comparison in that case will lead to interpretation about tasks' domain. Our goal is different. We seek to compare tasks in terms of skills, which refers to the conditional measure $\Prob_{Y|X}$, {\it i.e.} the skills to estimate $Y$ given $X$. To clarify, we make the following assumptions:
\begin{itemize}[leftmargin=*]
    \setlength\itemsep{0em}
    \item (H1) All our tasks are probability measures on the same space $(\Xspace \times \Yspace)$ which is true in the generative paradigm where $\Xspace$ and $\Yspace$ are both texts.
    \item (H2) For all tasks, the marginal distribution $\Prob_X$ will always remain the same. This is equivalent to considering that all our tasks are performed on the same input text\footnote{We provide more details about this hypothesis in~\autoref{app:measure-theory:h2}}.
\end{itemize}
Given~\autoref{def:task} and the different hypotheses, solving a task will be considered as the estimation of the conditional probability measure $\Prob_{Y|X}$. Then, we define the inclusion between two tasks as the answer to the following question: 
Given two tasks $\Prob_{XY_U}$ and $\Prob_{XY_V}$, does the estimation of $\Prob_{Y_U|X}$ implies being able to estimate $\Prob_{Y_V|X}$?
This question gives the simplest idea of the inclusion between two tasks: if we can perform one task, we can perform the other one. However, this is too restrictive since the estimation of $\Prob_{Y_U|X}$ can effectively not directly imply having the entire measure $\Prob_{Y_V|X}$. However, having $\Prob_{Y_U|X}$ can give strong hints (information) on the shape of $\Prob_{Y_V|X}$~\cite{boudiafUnifyingMutualInformation2021a} and we would like to capture this situation as an inclusion one\footnote{In~\autoref{app:information-theory:ce-decomp} we detail why we can have this situation}. This is the reason why we define a relaxed version of the inclusion,
\begin{defn}[Lenient-inclusion]\label{def:relaxed-inclusion}
    Given two tasks $\Prob_{XY_U}$ and $\Prob_{XY_V}$, we say that $\Prob_{XY_V}$ is included into $\Prob_{XY_U}$ (denoted as $\Prob_{XY_V} \lenincludes \Prob_{XY_U}$), iff the estimation of $\Prob_{Y_U|X}$ is informative about $\Prob_{Y_V | X}$.
\end{defn}
In~\autoref{def:relaxed-inclusion}, the notion of informativeness depends on the context. The goal in the following will be to define different versions of the informativeness of one task about another.
\autoref{def:relaxed-inclusion} seems more simple and requires less constraints on the shapes of $\Prob_{Y_U|X}$ and $\Prob_{Y_V|X}$ and we'll stick to this definition. Moreover, we can refer to an intensity of the inclusion which refers to how informative one task is about another. Still, to find inclusion, we must be able to manipulate very complex probability measures which are here our tasks\footnote{In~\autoref{app:measure-theory} arguments are developed, with the only use of measure theory.}. Thus, a meaningful and tractable representation of tasks is needed to address the inclusion estimation. Such a representation can be obtained by looking at a model fine-tuned to solve the task. In fact, it has been shown in~\cite{boudiafUnifyingMutualInformation2021a,achilleEmergenceInvarianceDisentanglement2018,tishbyInformationBottleneckMethod2000} that models fine-tuned to solve a task produce sufficient statistics of the task, in Fisher's sense~\cite[Definition 3.2 p.43]{KEENER}\footnote{For more details about this we refer to~\autoref{app:information-theory:model-suff}}. In the case of current language models, these sufficient statistics are essentially the continuous representations (or embeddings) $\Emb \in \EmbSpace$ of text $X$ produced by the models. These embeddings will thus be used as proxies to estimate task inclusion. In the following, for sake of simplicity we will refer to $\Prob_{XY_U}$ and $\Prob_{XY_V}$ as respectively task $U$ and task $V$.

\begin{figure}
    \centering
    \includegraphics{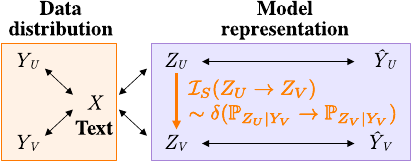}
    \caption{Illustration of proposed task comparison framework. $X$ is textual input, $Y$ is reference output, $\hat{Y}$ is system output, $Z$ represents embeddings, ($U$, $V$) is a pair of tasks. $\delta()$ is statistical deficiency and $\Irond_S$ is the information sufficiency proxy.}
    \label{fig:task-comm-comp}
\end{figure}

\subsection{Deficiency: a notion of inclusion}

\autoref{def:relaxed-inclusion} states that $V \lenincludes U$ if and only if solving the task $U$ is informative about the task $V$. One way to verify this, is by comparing embeddings' distributions of models fine-tuned on respectively $U$ ($\Emb_U$) and $V$ ($\Emb_V$). In the following, we will use the notation $\Prob_{\Emb_U | Y_U} \in \Mrond(\EmbSpace | \Yspace)$. This kernel describes the distribution of $\Emb_U$ given the values $Y_U$ takes. A good kernel $\Prob_{\Emb_U | Y_U}$ would cluster the embeddings $\Emb_U$ depending on the value $Y_U$ can take (similar values of $Y_U$ lead to similar $\Emb_U$), assuring we can infer $Y_U$ from $\Emb_U$\footnote{We refer to~\autoref{app:formalism:kernel-emb} for more details.}. Then, following the path of inclusion of $V$ into $U$, a question of interest would be: how good is the kernel $\Prob_{\Emb_U | Y_V}$? Or equivalently: how clustered $\Emb_U$ is relatively to $Y_V$? Statistical deficiency, first introduced by~\citet{blackwellComparisonExperiments1951,blackwellEquivalentComparisonsExperiments1953} in the context of comparison of statistical experiments, provides a set up to compare $\Prob_{\Emb_U | Y_V}$ and $\Prob_{\Emb_V|Y_V}$ (the latter being considered good for task $V$), providing one possible answer to previous questions.
\begin{defn}[Deficiency~\cite{camSufficiencyApproximateSufficiency1964}]\label{def:deficiency}
    Let $\Prob_{XY_U}$ and $\Prob_{XY_V}$ be two tasks, and $\Emb_U$ and $\Emb_V$ the embeddings of $X$ given by fine-tuned models. The deficiency $ \delta(\Prob_{\Emb_U | Y_V} \rightarrow \Prob_{\Emb_V | Y_V}) $ measures the informativeness of $\Prob_{Z_U | Y_V}$ about $\Prob_{Z_V|Y_V}$ and is defined as:
    \[
        \begin{split}
        \delta(&\Prob_{\Emb_U | Y_V} \rightarrow \Prob_{\Emb_V | Y_V} ) \\
        &\triangleq  \inf_{M\in\Mrond(\EmbSpace | \EmbSpace)}  \| M\!\circ\!\Prob_{\Emb_U | Y_V} - \Prob_{\Emb_V | Y_V} \|_{\TV}. \\
        \end{split}
    \]
    If $\delta(\Prob_{\Emb_U | Y_V} \rightarrow \Prob_{\Emb_V | Y_V})=0$ (no deficiency), we say that $\Prob_{\Emb_U | Y_V}$ is {\bf sufficient} for $\Prob_{\Emb_V | Y_V}$. Deficiency is a quantity in $[0,1]$ which quantifies how informative $\Prob_{\Emb_U | Y_V}$ is about $\Prob_{\Emb_V|Y_V}$ ($0$ being perfectly informative).
\end{defn}

\begin{thm}[$0$-deficiency]\label{thm:info-process}
     \[
        \delta(\Prob_{\Emb_U | Y_V} \rightarrow \Prob_{\Emb_V | Y_V} ) = 0 \Rightarrow V\lenincludes U.
    \]
\end{thm}

The proof of this result is given in~\autoref{s:thm1-proof}. Restricting our definition of inclusion to $0$-deficiency pairs of task is too restrictive and can rarely be achieved in practice, due to properties of the $\TV$ distance. We thus study the continuous spectrum of informativeness measured by the deficiency. We leverage additional results due to~\citet{camSufficiencyApproximateSufficiency1964,lecamComparisonExperimentsaShort1996} that control the amount of information a task reveal about the other, function of the deficiency.

\begin{thm}[$\varepsilon$-deficiency~\cite{camSufficiencyApproximateSufficiency1964}]\label{thm:epsilon-def}
    Let $\varepsilon > 0$. Then, $\delta(\Prob_{\Emb_U|Y_V} \rightarrow \Prob_{\Emb_V|Y_V} ) < \varepsilon$ if and only if, for any bounded loss function $\ell$, we have,
    \[
        \Rrond_\ell(Y_V, \Emb_U) - \varepsilon \leq \Rrond_\ell(Y_V, \Emb_V) .
    \]
    Where, $\Rrond_\ell(Y_V, \Emb_U)$ denotes the statistical risk of inferring $Y_V$ from $\Emb_U$ measured with the loss function $\ell$
\end{thm}

\autoref{thm:epsilon-def} guarantees that the lower the deficiency the more $V$ is included into $U$. The deficiency can be seen as the {\it missing} information.

\subsection{Inclusion estimation}

Deficiency proposed in~\autoref{def:deficiency} is intractable in practice due to the complexity of $\TV$ distance~\cite{bhattacharyyaApproximatingTotalVariation2023}. However, Information Sufficiency (IS), originally introduced by~\citet{arimotoInformationtheoreticalConsiderationsEstimation1971} and more recently used in~\cite{xu2020theory,darrin$textttCOSMIC$MutualInformation2024,darrinWhenEmbeddingModel2024}, can be an interesting proxy to estimate deficiency. IS of $Z_U$ relatively to $Z_V$, denoted as $\Irond_S(Z_U \rightarrow Z_V)$, is a lower bound of the Mutual Information (MI)~\cite[Section 2.3 p.20]{Cover91} $I(Z_U; Z_V)$ between embeddings, and is defined as,
\begin{equation}
    \begin{split}\label{eq:info_suf_simple}
        \Irond_S(Z_U \rightarrow Z_V) \triangleq \hat{h}(Z_V) - \hat{h}(Z_V | Z_U),
    \end{split}
\end{equation}
where $\hat{h}$ denotes an estimation of the entropy and the conditional entropy\footnote{For more details about IS, we refer to~\autoref{app:formalism:info-suff-details}}. Moreover, MI is a quantity that has links to Lenient inclusion of~\autoref{def:relaxed-inclusion} thanks to the following relations\footnote{We provide a proof of these relations in~\autoref{app:information-theory:im-inclusion}},
\begin{equation}
    \begin{split}
        I(Z_U ; Z_V) &\leq I(Z_U; Y_V), \\
        I(Z_U; Z_V) &\leq I(Z_V; Y_U), \\
    \end{split}
\end{equation}
where $I(Z_U; Y_V)$ can be seen as the \say{information} embeddings trained on $U$ have on the response of $V$ which is related to~\autoref{def:relaxed-inclusion}. Then in terms of interpretation, the larger $\Irond_S(Z_U \rightarrow Z_V)$ is, the more $Z_U$ is informative about $Z_V$ thus the more information $U$ carries about $V$, then the more included $V$ is into $U$. Moreover, IS has been empirically tested as an interesting proxy to evaluate deficiency between embeddings~\cite{darrinWhenEmbeddingModel2024}.
Lastly, we recall that IS is an unbounded positive value, ranging from $0$ to $+\infty$ contrary to deficiency which lies in $[0,1]$\footnote{For more details, we refer to~\autoref{app:formalism:info-suff-details}.}.
In the following, we will use $\Irond_S(Z_U \rightarrow Z_V)$ as a measure of how much task $V$ is included in task $U$. The higher the value of $\Irond_S$ the more $V$ will be considered as included into $U$. 
The main idea behind deficiency and its estimation through IS is to be able to simulate $Z_V$ from $Z_U$. This has also been explored by~\citet{lange-etal-2021-share} in the context of task relationship, but with the use of the $L_2$ distances between samples of the measures with only a linear transformation allowed. Our setup is more general and allows for broader class of transformations with a reconstruction loss connected to the inclusion notion ({\it c.f.}~\autoref{thm:epsilon-def}).

\section{Experimental setup}
\label{sec:experimental_setup}
Two different settings are proposed here to quantify inclusion between tasks using the estimation of IS ({\it c.f.}~\autoref{eq:info_suf_simple}) between embeddings. We will focus on inferring relations between pairs of tasks $U$ and $V$ using the following reasoning:
\[
    \Irond_S(Z_V \rightarrow Z_U)\leq \Irond_S(Z_U\rightarrow Z_V)  \Rightarrow V \lenincludes U.
\]

\paragraph{Synthetic experiment.} First, we consider a synthetic example to explore and validate the tool of IS as an inclusion metric. Task data is generated from a Hidden Markov Model (HMM)~\cite{baum-jstor-96} over a vocabulary of size $10$, from which we sample sequences of up to $30$ symbols. A standard 1-layer transformer-based language model is trained on these synthetic samples in a next-token prediction fashion. This {\it foundation model} is then fine-tuned on three simple classification tasks for which inclusion relationships are known. Given an input~$X$ made of a sequence $S = [s_1, \dots, s_n]$ generated by the HMM, and two characters $C_1$ and $C_2$ drawn from the vocabulary, the three tasks are:
\begin{itemize}[leftmargin=*]
    \setlength\itemsep{-0.2em}
    \item \texttt{First}$(S, C_1, C_2)$ (noted as {\tt F}) should return $Y=0$ if $C_1$ is the same as the first character of $S$ ($C_1 = s_1)$, $1$ otherwise.
    \item \texttt{Last}$(S, C_1, C_2)$ (noted as {\tt L}) should return $Y=0$ if $C_2$ is the same as the last character of $S$ ($C_2 = s_n)$, $1$ otherwise.
    \item \texttt{First\_or\_Last}$(S, C_1, C_2)$ (noted as {\tt F$\lor$L}) should return $Y=0$ if $C_1$ is the first character of $S$ and $C_2$ is the last ; else $1$ if $C_1$ is the first, $2$ if $C_2$ is the last ; $3$ otherwise.
\end{itemize}
It is straightforward to see that the task {\tt F$\lor$L} includes the others while the converse is not true. For all tasks considered, representation $X$ is only the output of the attention layer\footnote{Last token of the sentence, since the model is causal.}. We generated 11 datasets based on various HMMs with a change in the emission of the Markov chain (leading to 33 models). All presented results are averaged over these different datasets. For further details about the setup of this experiment, we refer to~\autoref{app:training:markov}.

\paragraph{NLP Pipeline.} The second experimental setup serves as a proof of concept, demonstrating that our inclusion measure can be effectively applied to NLP tasks. We selected five classification tasks common in linguistic pipelines, syntactic parsing (SYN), semantic role labeling (SRL), named entity recognition (NER), and coreference resolution (COR), along with a text generation task: summarization (SUM). 
These tasks were chosen because their inclusion relationships can be linguistically defined through annotation schemes. We used the OntoNotes dataset~\cite{onto-notes-dataset}, which provides multi-level linguistic annotations for news documents. SYN is based on the Penn Treebank~\cite{marcus1993building} annotation scheme, SRL on the Penn PropBank scheme, and NER uses eight OntoNotes categories for proper nouns (e.g., event, organization, person). COR annotations link coreferring noun phrases and, in some cases, verb phrases, making it a challenging task. For summarization, since OntoNotes lacks this annotation, we used {\tt GPT 3.5} to generate summaries for each document. 
%
From a linguistic perspective, annotation schemes exhibit direct inclusion relationships between syntactic (SYN) and semantic annotations (SRL). In contrast, the relationship between NER and both SYN and SRL is more interdependent: on one side, NER can leverage SYN and SRL to identify the left boundary of named entity spans; on the other side, NER can help defining noun phrases for SYN and SRL analyses.
COR relies on cues from all other levels. Similarly, summarization requires syntactic and semantic simplification as well as structural compression, making knowledge of linked entities crucial. Our goal is to verify that IS can uncover these relationships.
To standardize task processing, all tasks are reformulated as generative tasks. While summarization is inherently generative, linguistic tasks were adapted to extract relevant patterns under task-specific constraints. In order to simplify tasks and reduce sequence lengths, we restrict some tasks to generating a subset of the annotations (such as finding subjects for SYN), assuming that understanding the whole linguistic phenomenon is required to correctly generate such subset. \autoref{tab:data-desc} provides a detailed description of the reformulation; examples are listed in~\autoref{app:data}. All tasks were applied at the document level. We used 1297, 98, and 97 documents as train, validation and test sets, from the broadcast news and newswire subsets of Ontonotes. We compute performance and IS across tasks on Mistral 7B~\cite{mistral-model}, and Llama 3 8B~\cite{llama-3-model} in their Instruct and Base versions. This choice is based on the relatively good results these models offer for their size, which allows repeated fine-tuning within a reasonable resource budget. These models were fine-tuned using Low Rank Adaptations (\lora)~\cite{lora} of rank 8, a regularization coefficient $\alpha$ of 16 and a learning rate of 4e-5 with a constant learning rate scheduler. This choice of adaptation is based on the low amount of data we have for  training ($\approx$ 1300) and the good properties \lora\ offer in a case of low amount of data~\cite{fu-aaai-23}. Training is run for six epochs with a {\it best evaluation loss selection} strategy at the end of the training procedure. Fine-tunings were performed using the \href{https://huggingface.co/docs/transformers/index}{\tt transformers}~\cite{transformers-library} and \href{https://huggingface.co/docs/peft/index}{\tt peft} libraries from Huggingface.

\begin{table}
    \centering  
    {\small   
    \begin{tabular}{l l}
        \toprule
        {\bf Task} & {\bf Description} \\
        \midrule
        SYN & List all {\tt NP-SBJ} syntagms\\
        SRL & List all predicates {\tt PRED(ARG0,ARG1)} \\
        NER & List all NEs for {\tt each category} \\
        COR & List all coreference chains \\
        SUM         & Summarize \\
        \bottomrule
        \end{tabular}
    }
    \caption{Generative adaptation of NLP pipeline tasks}\label{tab:data-desc}
\end{table}

\section{Results}
\label{sec:results}
\label{sec:res}

\paragraph{Synthetic experiment.} In the context of the experiment on HMM data, we propose a deep analysis of the information sufficiency (IS). We provide on~\autoref{tab:markov-mi} IS results. The first thing that can be noticed is the presence of high mutual information in the diagonal of the table, which is an important sanity check: the most informative task for one task is the task itself. The second thing we can notice is that the following property holds,
\[
\begin{split}
    \Irond_S(\text{\tt F} \rightarrow \text{\tt L}) \leq \Irond_S(\text{\tt F$\lor$L} \rightarrow \text{\tt L}) , \\
    \Irond_S(\text{\tt L} \rightarrow \text{\tt F}) \leq \Irond_S(\text{\tt F$\lor$L} \rightarrow \text{\tt F}). \\    
\end{split}
\]
Which is in line with the relations between our tasks. Moreover, we have,
\[
    \Irond_S(\text{\tt F} \rightarrow \text{\tt F$\lor$L}) \approx \Irond_S(\text{\tt L} \rightarrow \text{\tt F$\lor$L}),
\]
which is an interesting results considering that tasks {\tt F} and {\tt L} carry the same amount of information on {\tt F$\lor$L}. The converse observation can be made,
\[
   \Irond_S(\text{\tt F$\lor$L} \rightarrow \text{\tt F}) \approx \Irond_S(\text{\tt F$\lor$L} \rightarrow \text{\tt L}),  
\]
with a similar justification. These results suggest that IS provides an interesting proxy to estimate information about relations between tasks.

\begin{table}
    \centering
    {\small \begin{tabular}{ c | c c c}
    \toprule
    & {\tt F} & {\tt F$\lor$L} & {\tt L} \\
    \midrule
    {\tt F} & 0.736 & 0.236 & 0.130 \\
    {\tt F$\lor$L} & 0.188 & 0.842 & 0.175 \\
    {\tt L} & 0.123 & 0.223 & 0.715 \\
    \bottomrule
    \end{tabular}}
    \caption{$\Irond_S(\text{row}\rightarrow \text{col})$ on pairs of synthetic tasks.}
    \label{tab:markov-mi}
\end{table}


\begin{figure}[t!]
    \centering
    \includegraphics[width=0.6\linewidth]{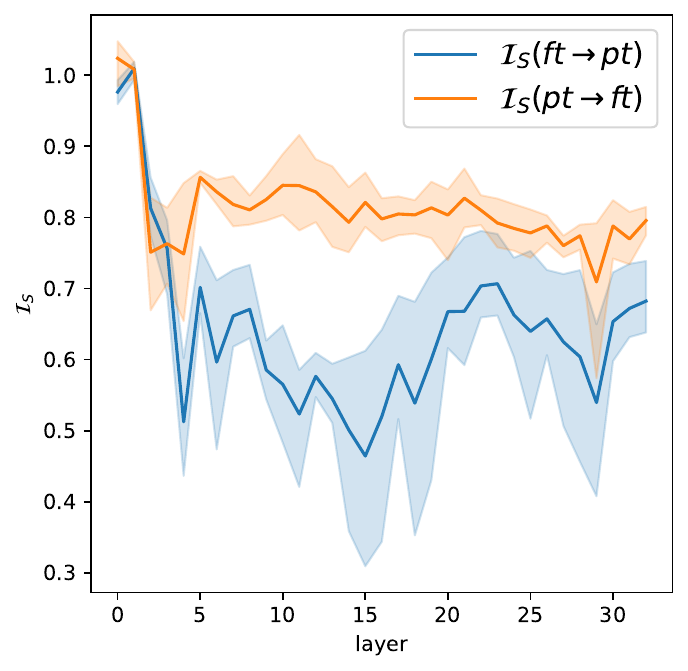}
    \caption{Layerwise information sufficiency between Mistral 7B base and that model model finetuned, averaged over the NLP pipeline tasks.}
    \label{fig:mistral-base-pretrained-vs-finetuned}
\end{figure}

\paragraph{NLP Pipeline.} \autoref{tab:perf-res} presents the results for each model across tasks, evaluated not on direct parser performance but on the generation process using RougeL scores~\cite{lin-2004-rouge} (as described in Table~\ref{tab:data-desc}). As expected, coreference resolution is the most challenging task. For summarization, all models perform similarly in terms of RougeL. Instruct models improve results for Llama but not for Mistral, though these differences are minor given the small test corpus. First, we focus on the strategy for determining IS between pairs of tasks, which can be computed from embeddings at any layer of the models. It has been empirically shown that different layers do not encode the same knowledge in LLMs: middle layers tend to focus more on the task the model is trained to solve, while last (deepest) layers focus more on the generation format of the task~\cite{siddiqui2024deeper,zhang-etal-2024-investigating,fischer2024large}. To support this fact, we ran a simulation for which we compared the hidden representations of the fine-tuned and the pre-trained model. More specifically, we compared $\Irond_S(Z_U \rightarrow Z)$ to $\Irond_S(Z \rightarrow Z_U)$, where $Z$ refers to the embeddings of the pre-trained model. We report results in~\autoref{fig:mistral-base-pretrained-vs-finetuned} for the Mistral Base model (\autoref{fig:layer-selection} contains this figure for all models). 
Middle layers seem to stand out more from the pre-trained model, suggesting that they are the ones that seem to encode the task. This result is in line with recent work about redundancy between large language models' layers~\cite{zhao2024layer,huang2024large,gromov2024unreasonable,men2024shortgpt,gonzalez2025leveraging}. Given this observation, in the following we will only look at average IS over layers 10-15, which seems to be layers for which the gap with the pre-trained model is the most significant. We propose in~\autoref{app:layer-ablation} an ablation study about the use of different layers, where we show the ineffectiveness of deep layers for task modelisation.
In~\autoref{fig:mean-mi}, we plot a heat-map of the average IS across the four models\footnote{Considering the multiple means, it is hard to provide proper confidence intervals.}. First, summarization is probably the most elaborate task considering that IS from any linguistic task to summarization (column SUM) leads to the lowest values. The only informative task for summarization is itself. On the opposite we can see that SYN task seems to be included in every other task (high values in the associated columns). Moreover, we have,
\[
    \begin{split}
        \Irond_S(\text{SYN} \rightarrow \text{SRL}) &\leq \Irond_S(\text{SRL} \rightarrow \text{SYN}) \\ 
        \Irond_S(\text{SRL} \rightarrow \text{NER}) &\leq \Irond_S(\text{NER} \rightarrow \text{SRL}), \\
    \end{split}
\]
leading to the following lenient inclusion ranking, $\text{SYN} \lenincludes \text{SRL} \lenincludes \text{NER}$, which is completely in line with our premises about the linguistic pipeline. Finally, once again, COR task is the most challenging among linguistic tasks, as it seems to contain information about all linguistic tasks, while no other task seems to contain information about it (low values in the associated column). As for its comparison with summarization, we have,
\[
    \Irond_S(\text{COR} \rightarrow \text{SUM}) \leq \Irond_S(\text{SUM} \rightarrow \text{COR}),
\]
which is once again in line with our initial thoughts on the linguistic pipeline. By putting all together these interpretations, clear hints seem to appear about the existence of the linguistic pipeline in the space of tasks (at least in the sense of the Ontonotes annotations mapped to generative tasks).

\begin{table}
\centering
{\small \begin{tabular}{l c@{~}c c@{~}c}
\toprule 
 & \multicolumn{2}{c}{\bf Base} & \multicolumn{2}{c}{\bf Instruct} \\
 & {\bf Llama} & {\bf Mistral} & {\bf Llama} & {\bf Mistral} \\
\midrule
SYN & 97.6 & 97.5 & 97.6 & 97.3 \\
SRL & 81.5 & 80.5 & 82.0 & 81.8 \\
NER & 86.7 & 87.8 & 85.0 & 86.3 \\
COR & 53.9 & 61.2 & 53.7 & 61.7 \\
SUM & 48.8 & 49.6 & 49.6 & 48.5 \\
\bottomrule
\end{tabular}}
\caption{RougeL scores of LLMs on generative versions of NLP pipeline tasks. Rouge implementation comes from \href{https://huggingface.co/docs/evaluate/index}{evaluate} library.}\label{tab:perf-res}
\end{table}

\begin{figure}[t!]
    \centering
    \includegraphics[width=.9\linewidth]{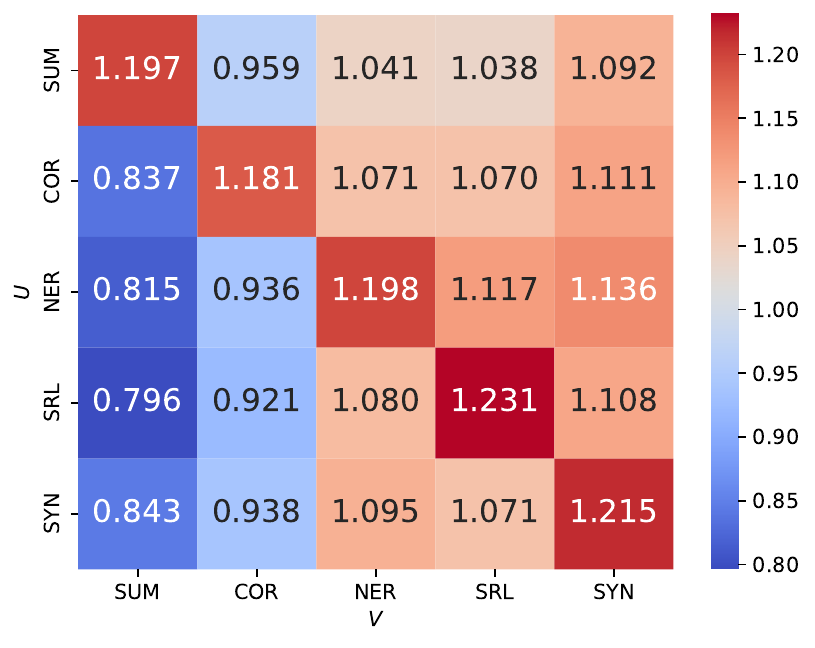}
    \caption{Average of $\Irond_S(\text{row} \rightarrow \text{col})$ across models.}
    \label{fig:mean-mi}
\end{figure}

\paragraph{Predictive power.} Interpretation of IS essentially relies on how much information one task contains about another {\it v.s.} how much information the others contain about the first one. Instead of having local interpretation of every combinations (which tends to increase rapidly) we can sum-up this idea by introducing a quantity we called the {\it predictive power} (PP). PP of a task $U$ can be defined as 
\[
    \text{PP}(U) \triangleq \sum_V \Irond_S(Z_U \rightarrow Z_V) - \Irond(Z_V \rightarrow Z_U).
\]
Interpretation is direct: the higher PP is for a task, the more it contains information about the others while the others do not contain information about it (which is essentially interpretations we made earlier). \autoref{tab:ranking} provides the increasing ranking of the different tasks in terms of predictive power for the different  models (the higher the ranking is the more PP of the task is high). On average (across models), tasks order in terms of PP seems to respect our premises about the NLP pipeline. Additionally, we can see that base models respect the pipeline's order, while some noise is present on Instruct models. This can be contrasted with the fact that Instruct models are pre-trained to perform a wide range of tasks, which can disrupt task modeling and thus ranking~\cite{mueller2024multi}.

\begin{table}[]
    \centering
    {\small 
        \begin{tabular}{l|c|c@{~}cc@{~}c}
        \toprule
              &       & \multicolumn{2}{c}{\bf Llama 3} & \multicolumn{2}{c}{\bf Mistral} \\
              & {\bf Avg.} & {\bf Base} & {\bf Instruct} & {\bf Base} & {\bf Instruct} \\
        \midrule
        SUM & 4.0 & 4 & 4 & 4 & 4 \\
        COR & 3.0 & 3 & 3 & 3 & 3 \\
        NER & 1.5 & 2 & 2 & 2 & 0 \\
        SRL & 0.75 & 1 & 0 & 1 & 1 \\
        SYN & 0.75 & 0 & 1 & 0 & 2 \\
        \bottomrule
        \end{tabular}
    }
    \caption{PP ranking of the tasks for different models. The higher the ranking, the more informative about the others is the task. Avg. refers to the average ranking.}
    \label{tab:ranking}
\end{table}

\section{Discussion}
\label{sec:discussion}

\paragraph{Synthetic experiment.} In~\autoref{tab:markov-mi} while some interesting interpretations can be done, problems remain as we have $
    \Irond_S(\texttt{F$\lor$L} \rightarrow .) \leq \Irond_S(. \rightarrow \texttt{F$\lor$L}),
$
which is not a desired property. However this property can largely be explained by the higher entropy of the hidden states of the task {\tt F$\lor$L} as shown in~\autoref{tab:markov-entropy}. This higher entropy will automatically give higher values of IS of the type $\Irond_S(.\rightarrow \texttt{F$\lor$L})$ which is basically illustrated on~\autoref{tab:markov-mi} (column {\tt F$\lor$L})\footnote{For more details, we refer to~\autoref{eq:info_suf}}. This higher entropy can be explained by the fact that {\tt F$\lor$L} is a 4-class classification problem while the others are 2-class.

\begin{table}
    \centering
    {\small \begin{tabular}{l | c c c}
    \toprule
        & F & F$\lor$L & L \\
    \midrule
         $H$ & 0.656 & 0.761 & 0.634 \\
    \bottomrule
    \end{tabular}}
    \caption{Entropy of the hidden states (synthetic exp.)}
    \label{tab:markov-entropy}
\end{table}

\paragraph{Task inclusion.} First, while we showed that we can mainly uncover the linguistic pipeline, we can observe similarity of our measures for tasks such as SRL and SYN ({\it i.e.} $\Irond_S(\text{SYN} \rightarrow \text{SRL}) \approx \Irond_S(\text{SRL} \rightarrow \text{SYN})$) thus questioning the significance. This result was expected in this case, as the SRL task considered here is closely related to the SYN task. To simplify the tasks and reduce sequence lengths, we restricted SRL and SYN to generating only a subset of annotations (ARG0 and ARG1 for SRL, and subjects and objects of verbs for SYN {\it c.f.}~\autoref{app:data}). As a result, the two tasks are naturally closely aligned, a fact that is further highlighted by our metrics.
Additionally, the setup we address in this work is rather simplistic. First, even among pipeline tasks, it is well known that there is no unidirectional relationship between tasks. For example, although semantic analysis is supposed to rely on syntactic analysis in linguistic pipelines, some syntactic constructs such as prepositional phrase attachment can only be disambiguated by semantic constraints~\cite{brill1994rule}. Although statistical deficiency only allows for strict inclusion, the mutual information proxy represented by IS seems to be more in line with what we know from linguistics. Second, prompts can be written to create arbitrary combination of tasks, which would allow for very diverse, yet controlled, instances of inclusion. Should inclusion resulting from a logical operator be differentiated from inclusion resulting from sequential application of instructions? Finally, the end goal is probably to decompose tasks in a minimal non-overlapping set of skills, a notion which has been eluded in all benchmarking efforts so far\footnote{In~\autoref{app:measure-theory} we formalize task decomposition.}.

\section{Conclusion}

This work aims at characterizing the space of NLP tasks through the notion of inclusion which we define as statistical deficiency. We propose IS as a tractable surrogate which allows comparing tasks from the embeddings of models trained on datasets annotated with those tasks. Experiments on synthetic and real NLP data suggest that this empirical notion of inclusion aligns with our preconception of task processing pipelines, potentially revealing which skills are required to perform some of the more elaborate tasks. 
Future work includes applying this framework to the selection of instruction tuning data by selecting most informative tasks/instructions within the data-mix to optimize dataset sizes without loss of performances. Another interesting line would be the conception of orthogonal evaluation benchmarks. We also plan on exploiting the task space structure for better handling task composition and generalization. A promising direction is to structure tasks as a Partial Ordering Set, that can be derived from a numerical application~\citep{peleg1970utility} and which~\citet{shannon-58} has applied to structure communication channels.

\section{Limitations}

One of the main limitation of this work relies in the use of information sufficiency to estimate task inclusion. As we stated, IS is used as proxy to estimate deficiency. However computing IS, contrary to deficiency, does not account for response values $Y_U$ and $Y_V$. These variables are part of the very essence of a task, as we can see from~\autoref{def:task}. A more accurate way of estimating inclusion would be to directly estimate deficiency which will be addressed in future work. 
Under certain hypothesis, we can use other more tractable distances~\cite{gibbs-02-wiley}, leading to other definitions of the deficiency. The other problem with this approach is that our inclusion estimate is empirical by nature, and relies on how representative the underlying corpus is of the general task distribution. This can be improved by collecting larger corpora but is fundamentally unbounded.

This leads to the second limitation of this work: estimating task inclusion from a single dataset, OntoNotes, in a single language, English. Hypothesis H2 requires that we use the same inputs for the compared tasks, and not many corpora offer this property. The tasks we cover are both limited in number and scope, only addressing a subset of the classic NLP pipeline, and are altered by framing them in a generative setting. In particular, considered NLP pipeline tasks involve only subsets of underlying linguistic structures, which weakens our task ordering claims. Besides, we only look at two average-sized LLMs given the combinatorics of model training/evaluation, and only use a single adaptation method, LoRA. Adaptation through fine-tuneing is one way of specializing models to perform a task, but we should explore zero-shot prompting and in-context learning as well (note that the proposed formalism is still valid in those cases).
Those experiments should be seen as a proof of concept, not a complete proof, to confirm that the intuition of the linguistic pipeline can be rediscovered through the proposed metric. Future work is needed to extend the scope of results.


\newpage
\bibliography{custom,biblio,bibliography/information-theory,bibliography/deficiency,bibliography/books,bibliography/additional_ressources,bibliography/MODEL_MERGE,bibliography/culture,bibliography/fine-tuning-study,bibliography/probing-studies,bibliography/TASKS,bibliography/principal-angles,bibliography/ref-rebutalls-round-1}

\onecolumn
\newpage
\appendix

\newpage

\section{Details about the formalism}
\label{app:formalism}

In this section we give additional details about~\autoref{formalism}.

\subsection{Measures and Kernels}\label{app:formalism:details}

In this study, we assume that all considered spaces are standard Borel~\cite{crauel-02}. Each such space $\Xspace$ is equipped with its Borel $\sigma$-algebra $\Brond(\Xspace)$. Having this regularity about the topology on our spaces is necessary to have equivalence between conditional probabilities and Markov transition kernels~\cite[Theorem 8.5 p.168]{kallenger-springer-22}, assuring the existence of conditional probabilities. Moreover, for every random variable $X\in \Xspace$, we denote by $\Prob_X$ the push-forward measure induced by $X$ on $\Xspace$. Thus considering two random variables $X \in \Xspace$ and $Y \in \Yspace$, we have $\Prob_{Y|X} \in \Mrond(\Yspace | \Xspace)$. We additionally recall here some basic properties on measures and Markov kernels, that are used during this study.
\begin{defn}[Total variation distance]\label{def:tv}
    Given two probability measures $P$ and $Q$ on $(\Xspace, \Brond(\Xspace))$,
    \[
        \|P-Q\|_{\TV} \triangleq \underset{b \in \Brond(\Xspace)}{\sup} | P(b)-Q(b)|.
    \]
\end{defn}
Total variation distance between probability measures induces a distance on Markov Kernels' space, which can be defined as following, for two kernels $M$ and $K$ in $\Mrond(\Yspace | \Xspace)$.
\[
   \|K-M\|_{\TV}=\underset{x \in \Xspace}{\sup} ~ \|K(.|x)-M(.|x) \|_{\TV}.
\]
Deficiency in~\autoref{def:deficiency}, uses a composition operation between Markov kernels, that is defined as following,
\begin{defn}[Markov composition operation]
    Let $K \in \Mrond(\Zspace | \Yspace)$ and $M \in \Mrond(\Yspace | \Xspace)$, 
    \[
        (K \circ M) (b | z) = \int_{\Yspace} K(b | y) M(dy | z)
    \]
    This composition operation must be viewed as a generalization of the law of total probability.
\end{defn}

\subsection{Embeddings as Kernels}\label{app:formalism:kernel-emb}

In this study, $\Prob_{\Emb_U | Y_U} \in \Mrond(\Zspace | \Yspace)$ is used to describe embeddings, especially in~\autoref{def:deficiency}. This kernel describes the shape of $\Emb_U$ given particular values of $y\sim \Prob_{Y_U}$. An interesting case to understand this kernel is classification tasks, for which $\Yspace$ is a finite discrete space. It is a well known fact that for classification tasks, $\Prob_{\Emb_U | Y_U}$ can be described as a cluster-type model. In that case $\Prob_{\Emb_U | Y_U}(.|y)$ corresponds to the cluster of $\Emb_U$ associated to $Y_U = y$. When looking at $\Prob_{\Emb_U | Y_V}$ we seek to understand if $\Emb_U$ has also a clustering-type structure which makes sens for $Y_V$ and then if we can infer values of $Y_V$ from $\Emb_U$, which is echoing the inclusion of $\Prob_{XY_V}$ in $\Prob_{XY_U}$.

\subsection{Proof of \autoref{thm:info-process}}
\label{s:thm1-proof}

We provide here the proof of~\autoref{thm:info-process}. This result, states that a $0$-deficiency can be seen as an inclusion between two tasks. As a recall the relaxed version of the inclusion of $V$ into $U$ states that solving task $U$ is informative to solve task $V$. \autoref{thm:info-process} states that this useful information obtained from one task for the other, is contained in the embeddings.

The proof of~\autoref{thm:info-process} requires another assumption about the fine-tuning operation. More particularly if $\Emb_U$ are the embeddings of a text $X$, provided by a model trained on the task $\Prob_{XY_U}$, then $Z_U$ can achieve the best error rate (Bayes risk) on task $\Prob_{XY_U}$, which is equivalent as,
\begin{equation}\label{eq:bayesian-asumption}
    \begin{split}
        \Rrond_\ell(Y_U, \Emb_U)& \triangleq \underset{d \in \mathcal{M}(\Yspace_U|\EmbSpace_U)}{\inf} \mathbb{E}_{y \sim Y_U} \mathbb{E}_{\hat{y} \sim d\circ \Prob_{\Emb_U|Y_U}(.|y)} \ell(y, \hat{y}) \\
        & = \underset{d \in \mathcal{M}(\Yspace_U|\Xspace)}{\inf} \mathbb{E}_{y \sim Y_U} \mathbb{E}_{\hat{y} \sim d\circ \Prob_{X|Y_U}(.|y)} \ell(y, \hat{y}),\\
    \end{split}
\end{equation}
for any bounded loss function $l$. Second infimum is achieved for the {\it posterior}, $\Prob_{Y_U | X}$, meaning that there exists $T_U \in \Mrond(\Yspace | \EmbSpace)$ such that,
\begin{equation}\label{eq:posterior-access}
    T_U \circ \Prob_{\Emb_U | X} = \Prob_{Y_U | X}.
\end{equation}
\begin{proof}
    If $\delta(\Prob_{\Emb_U|Y_V} \rightarrow \Prob_{\Emb_V|Y_V})=0$, then there exists some $K \in \Mrond(\EmbSpace | \EmbSpace)$, such that $K\circ \Prob_{\Emb_U|Y_V} = \Prob_{\Emb_V|Y_V}$. The statistical risk of the task $\Prob_{XY_V}$ from $Z_V$ being given by,
    \[
      \Rrond_\ell(Y_V, Z_V) = \underset{d \in \mathcal{M}(\Yspace|\EmbSpace)}{\inf} \mathbb{E}_{y \sim Y} \mathbb{E}_{\hat{y} \sim d\circ \Prob_{\Emb_V|Y_V}(.|y)} \ell(y, \hat{y}),
    \]   
    By data-processing~\cite{reid-jmlr-11,raginskyShannonMeetsBlackwell2011}, we thus have,
    \[
        \Rrond_\ell(Y_V, Z_U) \leq \Rrond_\ell(Y_V, Z_V)\quad \forall~\ell~ \text{s.t.}~\|l\|_{\infty} \leq 1.
    \]
    However, $Z_V$ is supposed to reach Bayes risk, implying thus that $Z_U$ also reaches Bayes risk for $Y_V$. Consequently there exists $T \in \Mrond(\Yspace | \EmbSpace)$ such that,
    \[
        T\circ \Prob_{\Emb_U | X} = \Prob_{Y_V | X}.
    \]
    Thus, by solving task $U$, {\it i.e.} by inferring the posterior $\Prob_{Y_U | X}$ through embeddings $Z_U$, we are also able to produce the posterior $\Prob_{Y_V | X}$ of the task $V$, which concludes the proof. 
\end{proof}

\subsection{Information Sufficiency definition}\label{app:formalism:info-suff-details}

As stated in the core of the study, while deficiency is the true measure of the inclusion, it is too complex to estimate forcing us to use information sufficiency. Information sufficiency is a lower bound of the mutual information between embeddings $Z_U$ and $Z_V$. We have,
\[
    \begin{split}
        I(Z_U; Z_V) &= H(Z_U) - H(Z_U | Z_V) \\
        &= H(Z_V) - H(Z_V | Z_U). \\
    \end{split}
\]
While we do not have access to the true distributions, calculus of the above entropy  is impossible. Information sufficiency estimates the distributions of the embeddings, with a maximum $\log$-likelihood estimation, by making the assumption that distributions lie in a parametric family. Information sufficiency has thus the following expression,
\begin{equation}
   \label{eq:info_suf}
   \begin{split}
        \Irond_{S}(Z_U \rightarrow Z_V )  &\triangleq \hat{H}_{\Prond_{\theta}(\EmbSpace)}(Z_V) - \hat{H}_{\Mrond_{\Theta}(\EmbSpace | \EmbSpace)}(Z_V | Z_U) \\
        \Irond_{S}(Z_V \rightarrow Z_U )  &\triangleq \hat{H}_{\Prond_{\theta}(\EmbSpace)}(Z_U) - \hat{H}_{\Mrond_{\Theta}(\EmbSpace | \EmbSpace)}(Z_U | Z_V),
   \end{split}
\end{equation}
where $\hat{H}_{\Prond_{\theta}(\EmbSpace)}$ is the estimation of the entropy on the class $\Prond_{\theta}(\EmbSpace)$. In our case we use KNIFE estimator~\cite{pichlerDifferentialEntropyEstimator2022} to compute this quantity. This estimator chooses Gaussian Mixtures for the parametric family. This choice can be justified by the fact that the set of Gaussian Mixtures is dense for the weak topology (convergence in probability) within the set of Lebesgue-continuous probability measures, assuring we can approximate every continuous probability measure with a Gaussian Mixture. Moreover, Gaussian mixtures have interesting properties in terms of smoothness. It has been shown in~\cite{donohoOneSidedInferenceFunctionals1988} and confirmed in~\cite{pichlerEstimationInformationMeasures2021} that estimating mutual information for distribution with no particular smoothness assumptions (particularly distributions respecting the {\it dense graph condition}) is impossible from a finite sample of the distribution (whatever the size), justifying once again the choice of Gaussian Mixture here.

One of the main justification of the fact that information sufficiency is an interesting proxy for estimating deficiency, lies in the fact that the term $\hat{H}_{\Mrond_{\Theta}(\EmbSpace | \EmbSpace)}(Z_V | Z_U)$ is an estimation of the amount of information, one embedding is carrying about the other and gives an interesting idea how, we can re-construct one embedding from another.

\section{A measure theoretic view of tasks}
\label{app:measure-theory}

In this work, we defined tasks as probability measures $\Prob_{XY}$ on a product space $(\Xspace \times \Yspace)$. In this section we propose interpretations of classical measure theory results to show links with the inclusion proposed in~\autoref{def:relaxed-inclusion}.

\subsection{About domain shift (H2 assumption)}\label{app:measure-theory:h2}

The objective of this study is to compare tasks in terms of the needed skills to solve them\footnote{The skills here refer to the conditional measure $\Prob_{Y|X}$}. To do so, we made the hypothesis H2 about the marginal distributions on the texts of our tasks. First, this assumption is also made in connected works~\cite{baoInformationTheoreticApproachTransferability2019} in particular set-ups this assumptions can be relaxed~\cite{tanOTCETransferabilityMetric2021}. We further justify here this hypothesis, and show that by fixing the marginal $\Prob_X$, when comparing two tasks (not through a fine-tuned model), we compare the skills needed to solve them. This is done through disintegration Theorem~\cite[Theorem 8.5 p.168]{kallenger-springer-22}. Given a task $\Prob_{XY}$, there exists a unique ($\Prob_X$ a-s) $M\in \Mrond(\Yspace |\Xspace)$, such that for every $E \in \Brond(\Xspace \times \Yspace)$,
\[
\begin{split}
    \Prob_{XY}(E) &= (\Prob_{X}\otimes M)~(E) \\
    &= \int_{\Xspace} \int_{\Yspace} \mathbb{1}_E(x,y) (\Prob_{X} \otimes M)~(dx, dy) \\
    &= \int_{\Yspace} \int_{\Xspace} \mathbb{1}_E(x,y) \Prob_{X}(dx) M(dy|x)
\end{split}
\]
Where $\otimes$, only denotes the product operator between measures defines on different measure spaces. In the following we will denote such kernel $M$ as $\Prob_{Y|X}$. This decomposition assures that whatever the task we are considering, we can disentangle the domain and the skills needed to solve the task,
\[
\Prob_{XY} = \underbrace{\Prob_{X}}_{\text{Domain}} ~ \otimes ~ \underbrace{\Prob_{Y|X}}_{\text{Skills}}
\]
Considering this disentanglement, we'll assume that comparison of tasks will not lead to interpretation about the domain. 

\begin{example}[Difference]
    Given two tasks $\Prob_{XY_U}$ and $\Prob_{XY_V}$, and $E \in (\Brond(\Xspace \times \Yspace))$,
    \[
        \Prob_{XY_U}(E) - \Prob_{XY_V}(E) = \int_{\Xspace} \int_{\Yspace} \mathbb{1}_E(x,y) \Prob_{X}(dx)(\Prob_{Y_U|X}(dy|x) - \Prob_{Y_V|X}(dy|x))
    \]
    In the case of same skills, this difference is equal to zero, and the value of the difference is strongly linked to differences in the skills required for tasks
\end{example}

\subsection{Measure theoretic view of inclusion}
\label{app:measure-theory}

In this Section we explore the possibilities to understand the behavior of a task, w.r.t another set of tasks, in a measure theoretic point of view {\it i.e.} by directly comparing tasks without using any proxy. 
A classical result in measure theory is the Lebesgue decomposition Theorem~\cite{brooksLebesgueDecompositionTheorem1971}, which states that one task can be decomposed interestingly by using another task, echoing a notion of shared information (and thus inclusion) between two tasks. The Theorem is the following: given two tasks, $\Prob_{XY_U}$ and $\Prob_{XY_V}$, we have,
\begin{equation}\label{eq:lebesgue-decomp}
    \Prob_{XY_U} = \mu_V + \rho\quad\text{such that}\quad\mu_V << \Prob_{XY_V}~~\text{and}~~\rho\perp\Prob_{XY_V}.
\end{equation}
In the case of such decomposition, $\mu_V$ can be seen of the \say{information} $\Prob_{XY_U}$ encodes on $\Prob_{XY_V}$. By doing such interpretation, we implicitly consider that domination between measures is a sort of inclusion metric. We show below that this interpretation can be true.

Given two tasks $\Prob_{XY_U}$ and $\Prob_{XY_V}$, such that $\Prob_{XY_V} << \Prob_{XY_U}$ then for every $E \in \Brond(\Xspace \times \Yspace)$, there exists a unique (up to a $\Prob_{XY_U}$ null measure space), positive real valued function $f$ (which is called the density), such that,
\begin{equation}\label{eq:domination-expression}
    \begin{split}
        \Prob_{XY_V}(E) &= \int_{\Xspace} \int_{\Yspace} \mathbb{1}_E(x,y) f(x,y) \Prob_{XY_U}(dx,dy) \\
        &= \int_{\Xspace} \int_{\Yspace} \mathbb{1}_E(x,y) f(x,y) \Prob_X(dx) \Prob_{Y_U|X}(dy|x)
        .
    \end{split}
\end{equation}
Having a domination relation, allows us to express the $\Prob_{XY_V}$-measure of every event $E$, with the use of $\Prob_{Y_U|X}$, which is related to~\autoref{def:relaxed-inclusion}, {\it i.e.} having $\Prob_{Y_U|X}$ is informative about $\Prob_{Y_V|X}$. In the situations where $\Prob_{XY_V} << \Prob_{XY_U}$ we can state the $\Prob_{XY_V} \subset \Prob_{XY_U}$.

\begin{rem}[Domination in practice]
    In practice, we do not have access to theoretical probability measures, but we have access to samples from this measure. We therefore do not work on the theoretical measures but on the quantized version of these measures that we denote by $\tilde{\Prob}_{XY_U}$ and $\tilde{\Prob}_{XY_V}$. In that case, we define $S_U$ and $S_V$ the samples from respectively $\Prob_{XY_U}$ and $\Prob_{XY_V}$. In this case a sufficient condition conditions to have $\tilde{\Prob}_{XY_V} << \tilde{\Prob}_{XY_U}$ is that, $S_V \subset S_U$.
\end{rem}

Thus domination is interestingly related to the notion of task inclusion. One can thus ask how to find a similar set up of the one proposed with the information sufficiency in this study, but with using domination. We propose here several definitions to be able to express one task with respect to other ones,

\begin{defn}[Independent tasks]
    Let $\set{\Prob_{XY_i},~i\in\set{1, \dots, N}}$ a finite set of tasks. We say that this set is independent iff,
    \[
        \Prob_{XY_i}\perp\Prob_{XY_j} \quad \forall i \neq j
    \]
\end{defn}

\begin{prop}\label{prop:measure-theory:task-decomp}
    Let $\Prob_{XY}$ a task, and $\set{\Prob_{XY_i},~i\in\set{1, \dots, N}}$ be an independent set of tasks, then we have the following decomposition exists,
    \[
        \Prob_{XY} = \sum_{i=1}^N \mu_i + \rho \quad \text{such that}\quad \mu_i << \Prob_{XY_i},\quad \text{and}\quad \rho \perp \Prob_{XY_i}\quad \forall i
    \]
    Moreover this decomposition is uniquely verified by the different tasks. $\rho$ is the remaining of the decomposition, {\it i.e.} the task we can't explain from our independent set of tasks.
\end{prop}

\begin{example}[Interpretation of the remaining $\rho$]
    Given two independent tasks $\Prob_{XY_U}$ and $\Prob_{XY_V}$, we construct a new task $\Prob_{XY}$, such that $Y = f(Y_U, Y_V)$, where $f$ is a (possibly randomized) transformation. Accordingly to~\autoref{prop:measure-theory:task-decomp}, we have,
    \[
        \Prob_{XY} = \mu_U + \mu_V + \rho.
    \]
    In this case $\rho$ can be viewed as the contribution of the application $f$.
\end{example}

However, in~\autoref{prop:measure-theory:task-decomp}, the term $\rho$ remains hard to interpret. It would be interesting to have $\rho=0$, meaning that we have a set of task that is {\bf complete} for our target task $\Prob_{XY}$

\begin{defn}[Complete set of tasks]
    \label{def:complet-set}
    We say that an independent set of tasks $\set{\Prob_{XY_i},~i\in\set{1, \dots, N}}$ is complete for a task $\Prob_{XY}$, when,
    \[
        \Prob_{XY} = \sum_{i=1}^N \mu_i \quad \text{such that}\quad \mu_i << \Prob_{XY_i}.
    \]
\end{defn}

\begin{rem}
    Let $\Prob_{XY}$ a task and $T\triangleq\set{\Prob_{XY_i},~i\in\set{1, \dots, N}}$ an independent set of tasks. 
    \[
        T~ \text{is complete for}~ \Prob_{XY} \Rightarrow \Prob_{XY} << \sum_i \Prob_{XY_i}
    \]
\end{rem}

\begin{proof}
    Let $T$ be complete for $\Prob_{XY}$. Thus,
    \[
    \Prob_{XY} = \sum_i \mu_i.
    \]
    Let $E \in \Brond(\Xspace\times\Yspace)$, such that $\sum_i\Prob_{XY_i}(E)=0$. Because we use non-signed measures, this implies $\Prob_{XY_i}(E)=0~\forall i$, and consequently $\mu_i(E)=0~\forall i$, by construction of the $\mu_i$. We thus have the following property,
    \[
    \forall E\in \Brond(\Xspace\times\Yspace),\quad \sum_i\Prob_{XY_i}(E)=0 \Rightarrow \Prob_{XY}(E)=0
    \]
    Which concludes the proof.
\end{proof}

Having a complete set of task for a target task $\Prob_{XY}$ means that given this set of task, we can fully determine this task. Having a complete independent set of task can be seen as having a {\bf basis of the task space}. From the basis some notions can be derived such as the dimension (cardinal of the basis) of the space. 
One can also define the complexity of a task, as the cardinality of the smallest set of independent task which is complete for the task (of course this set should not contain the task itself).

\section{Information theoretic view of fine-tuning}
\label{app:information-theory}

Information theory~\cite{Cover91} provides powerful tools to understand the dynamic of model trainings. We explore here this framework to make connections between learning theory and tasks inclusion. We first recall that given two random variables $Y$ and $Z$, $I(Y;Z)$ refers to the mutual information between these variables and is defined as,
\[
    \begin{split}
        I(Y;Z) &= H(Y) - H(Y|Z) \\
        &= H(Z) - H(Z|Y). \\
    \end{split}
\]

\subsection{Cross entropy decomposition}\label{app:information-theory:ce-decomp}

If we refer to~\autoref{fig:task-comm-comp}, and we suppose that our models are trained using cross entropy loss function $\Hrond$ (which is our case during all this study), then it has been shown~\cite{boudiafUnifyingMutualInformation2021a} that this loss function can be decomposed in two terms,
\[
    \Hrond(Y_U, \hat{Y}_U) = H(Y_U | \Emb_U) + D_{\text{KL}}(Y_U || \hat{Y}_U | \Emb_U).
\]
Since $H(Y_U)$ is a constant during  training, minimizing the cross entropy loss is equivalent to minimizing the following loss function:
\begin{equation}\label{eq:cross-entropy-decomp}
    \Lrond(Y_U, \hat{Y}_U) = \underbrace{- I(Y_U; \Emb_U)}_{\text{Task Info}} + \underbrace{D_{\text{KL}}(Y_U || \hat{Y}_U | \Emb_U)}_{\text{Task alignment}}.
\end{equation}
In this decomposition of the cross-entropy function the first term in mutual information refers to the information the model captures about the task {\it i.e.} the information the model captures to infer the response. The second term refers to the alignment of the estimation with the true labels. When training a model by minimizing the cross entropy we seek to create embeddings $\Emb_U$ of text $X$ that capture as much information about $Y_U$ as possible and that are aligned with the distribution $\Prob_{Y_U}$. Thus when a model is trained on a task $\Prob_{XY_U}$ it might not be able to perform the task $\Prob_{XY_V}$ due to a misalignment (second term) and thus it is not respecting the first task inclusion property as we defined it. However, the quantity $I(Y_V; \Emb_U)$ can be great suggesting that the model captures information about the task.

\begin{example}[Fully informative but miss-aligned]
    Let $\Prob_{XY_U}$ a task, and $Z_U \in \mathbb{R}^d$ the representation of the text $X$ given by a model trained using $\Hrond$. We additionally suppose that $\hat{Y}_U = f(Z_U)$, where $f$ is optimized w.r.t. $Y_U$. Let $\phi$ be a permutation in $\set{1, \dots, d}$. Because $\phi$ is an invertible mapping,
    \[
        I(Y_U; Z_U) = I(Y_U; \phi(Z_U)).
    \]
    However, we have $\hat{Y}^\phi_U = f(\phi(Z_U))$ and thus we have,
    \[
        D_{\text{KL}}(Y_U || \hat{Y}_U | Z_U) \leq D_{\text{KL}}(Y_U || \hat{Y}^{\phi}_U | \phi(Z_U))
    \]
    Thus based on $Z_U$ we can construct $\phi(Z_U)$, which is as informative as $Z_U$ for the task $\Prob_{XY_U}$, while it will perform worse due to a misalignment.
\end{example}

\begin{rem}[Probing interpretation]
    From~\autoref{eq:cross-entropy-decomp}, we can make a remark about current probing methods, which mainly consist of training a linear probe on top of the representation $Z_U$ without modifying them. Thus, when we probe a representation $Z_U$ on the task $\Prob_{XY_V}$, using cross entropy, we only optimize the following quantity
    \[
        D_{\text{KL}}(Y_V || \hat{Y}_V | Z_U),
    \]
    which is not directly linked to the information $Z_U$ captures about the task $\Prob_{XY_V}$. Probing simply answers the question of whether $Z_U$ can be aligned with $Y_U$ using a linear extractor, which can be restrictive. Current interpretation of probing would be correct if one uses sufficiently powerful probe as stated in~\cite{pimentel-arxiv-20}.
\end{rem}

\subsection{Sufficiency of the models}\label{app:information-theory:model-suff}

As we can see in~\autoref{eq:cross-entropy-decomp}, when fine-tuning a model on a task $\Prob_{XY_U}$ with cross entropy loss (which is our case in this study), the model constructs embeddings $\Emb_U$ such that $I(Y_U; \Emb_U)$ is maximized~\cite{boudiafUnifyingMutualInformation2021a}. However, by {\it data processing inequality}~\cite[Chapter 2, Section 8, p.34]{Cover91} we have the following inequality,
\[
    I(Y_U; \Emb_U) \leq I(Y_U; X).
\]
Thus maximizing the mutual information $I(Y_U; \Emb_U)$ is equivalent in approximate $I(Y_U;X)$. Thus the fine-tuning process tends to produce embeddings $\Emb_U$ such that,
\[
    I(Y_U; \Emb_U) \approx I(Y_U; X).
\]
This is equivalent in saying that fine-tuning a model on a task $\Prob_{XY_U}$ produces embeddings that are sufficient statistics (in Fisher's sense) of $X$ for $Y_U$~\cite{achilleEmergenceInvarianceDisentanglement2018}. This justifies the choice of using the embeddings as a proxy to represent the tasks, and proceed to the inclusion calculation.

\subsection{Mutual Information: an interesting proxy}\label{app:information-theory:im-inclusion}

We give additional results that justify the use of the mutual information as a proxy to estimate the task inclusion. \autoref{eq:cross-entropy-decomp} gives an interesting decomposition of the cross-entropy loss function giving the interpretation that the inclusion measure of a task through the study of a model, can be viewed as the estimation of $I(Y_V; \Emb_U)$ (resp. $I(Y_U; \Emb_V)$). However,~\autoref{fig:task-comm-comp} can be reduced to the following Markov chains $\Emb_U \leftrightarrow Y_V \leftrightarrow \Emb_V$, $\Emb_V \leftrightarrow Y_U \leftrightarrow \Emb_U$. These relations are true in the case of single reference tasks ({\it i.e.} the case where for an input data $x\sim\Prob_X$ there is only one single possible answer $y\sim\Prob_Y$), which is our case for any all linguistic tasks defined. Thus by {\it data processing inequality}, we have the following relations,
\[
    \begin{split}
        I(\Emb_U; \Emb_V) \leq I(Y_V; \Emb_U), \\
        I(\Emb_U; \Emb_V) \leq I(Y_U; \Emb_V).
    \end{split}
\]
Thus the information sufficiency as a lower bound of the true mutual information, is a lower bound of $I(Y_V;\Emb_U)$ and $I(Y_U; \Emb_V)$ which can be viewed as inclusion estimation for models fine-tuned using cross-entropy.

\section{Details about the trainings}
\label{app:training}

\subsection{Synthetic experiment}\label{app:training:markov}

As stated in this study we provided an experiment on a synthetic formal language generated from Hidden Markov Models (HMM). We produced 11 datasets following different HMM distributions. To generate different datasets, we fixed the underlying automaton of the HMM and we only changed the emission probabilities. The used HMM are simple ones, described on~\autoref{fig:markov-chain}. Presented results in this study are averaged on the different datasets we used. We provide on~\autoref{tab:markov-training-hparams} the parameters we used for pre-training and fine-tuning on HMM generated data. First of all to check that the pre-training of our transformer based language models has worked we compared the forward likelihood of the HMM to the estimated likelihood of the transformer model. Results are provided on~\autoref{fig:transformer-vs-forward}. We can clearly see that the pre-trained transformer model approximate better the forward distribution compared to a non-trained model.

\begin{figure}
    \centering
    \begin{tikzpicture}[-> , >= stealth' , shorten >=2 pt, auto, line width=0.5 pt, node distance=2cm ]
        \node[circle,draw] (one) {1};
        \node[circle,draw] (two) [right of=one] {2};
        \node[circle,draw] (three) [right of=two] {3};
        \node[circle,draw] (four) [right of=three] {4};
        \node[circle,draw] (five) [right of=four] {5};
        \node[circle,draw] (six) [right of=five] {6};
    
        \node[circle,draw] (pone) [above of=one] {$\mathbb{P}_{1,\mathcal{A}}$};
        \node[circle,draw] (ptwo) [above of=two] {$\mathbb{P}_{2,\mathcal{A}}$};
        \node[circle,draw] (pthree) [above of=three] {$\mathbb{P}_{3,\mathcal{A}}$};
        \node[circle,draw] (pfour) [above of=four] {$\mathbb{P}_{4,\mathcal{A}}$};
        \node[circle,draw] (pfive) [above of=five] {$\mathbb{P}_{5,\mathcal{A}}$};
        \node[circle,draw] (psix) [above of=six] {$\mathbb{P}_{6,\mathcal{A}}$};
    
        \path (one) edge [loop below] node{$0.5$} (one) ;
        \path (two) edge [loop below] node{$0.5$} (two) ;
        \path (three) edge [loop below] node{$0.5$} (three) ;
        \path (four) edge [loop below] node{$0.5$} (four) ;
        \path (five) edge [loop below] node{$0.33$} (five) ;
        \path (six) edge [loop below] node{$1.$} (six) ;
    
        \path (five) edge [bend left=40] node {$0.33$} (one) ;
        \path (one) edge node{$0.5$} (two) ;
        \path (two) edge node{$0.5$} (three) ;
        \path (three) edge node{$0.5$} (four) ;
        \path (four) edge node{$0.5$} (five) ;
        \path (five) edge node{$0.33$} (six) ;
    
        \path (one) edge (pone) ;
        \path (two) edge (ptwo) ;
        \path (three) edge (pthree) ;
        \path (four) edge (pfour) ;
        \path (five) edge (pfive) ;
        \path (six) edge (psix) ;
    \end{tikzpicture}
    \caption{Illustration of the used markov chain for data generation. The quantity $\Prob_{i, \Arond}$ refer to the emission probabilities of each states.}
    \label{fig:markov-chain}
\end{figure}
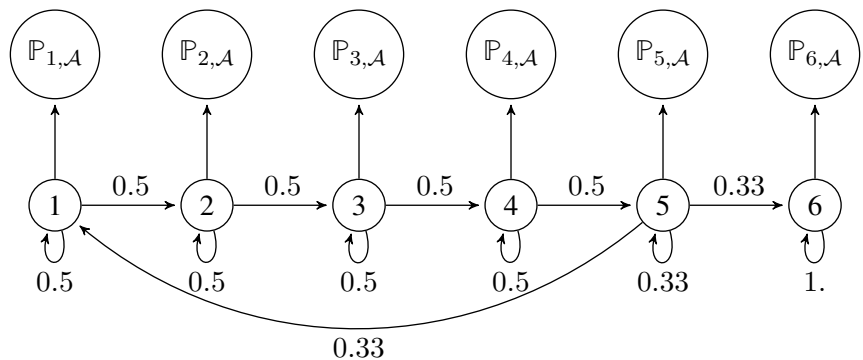

\begin{table}
    \centering
    \begin{tabular}{l | c c}
        \toprule
         & {\bf Pre-training} & {\bf Fine-tuning} \\
         \midrule
         input size & 50 & 50 \\
         hidden size  & 100 & 100 \\
         \# layer & 1 & 1 \\
         \# heads & 1 & 1 \\
         \hline
         lr & 2e-03 & 2e-04 \\
         lr-scheduler & cosine & cosine \\
         \# epochs & 100 & 100 \\
         \# batch & 200 & 200 \\
         \bottomrule
    \end{tabular}
    \caption{Description of the hyper-parameters for the training of transformer based models on described HMM.}
    \label{tab:markov-training-hparams}
\end{table}

\begin{table} 
    \centering
    \begin{tabular}{l c c c}
    \toprule
     & {\bf F1 Micro} & {\bf F1 Macro} & {\bf Acc} \\
    \midrule
    {\tt F} & 0.81 ({\small 0.16}) & 0.78 ({\small 0.20}) & 0.81 ({\small 0.16}) \\
    {\tt F$\lor$L} & 0.61 ({\small 0.20}) & 0.54 ({\small 0.27}) & 0.61 ({\small 0.20}) \\
    {\tt L} & 0.85 ({\small 0.11}) & 0.82 ({\small 0.15}) & 0.85 ({\small 0.11}) \\
    \bottomrule
    \end{tabular}
    \caption{Mean-accuracy (over the different datasets) of the different classification tasks.}
    \label{tab:markov-model-perf}
\end{table}

\begin{figure}
    \centering
    \includegraphics[width=\linewidth]{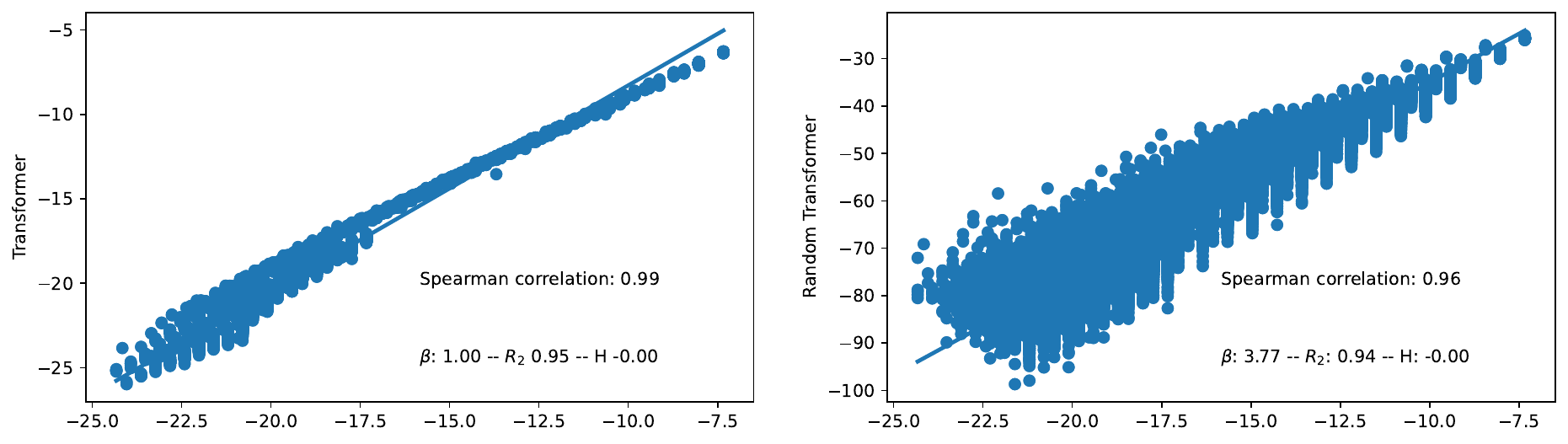}
    \caption{HMM forward likelihood {\it v.s.} empirical likelihood of the transformer based model}
    \label{fig:transformer-vs-forward}
\end{figure}

\subsection{Information Sufficiency estimation}

We provide on~\autoref{tab:knife-hparams} hyper-parameters of the KNIFE estimator. The only change was made for the production of graphics on~\autoref{fig:layer-selection} for which we limited our training on 20 epochs for sake of computational cost efficiency, since the goal of this simulation was mainly to understand the role of each layer.

\begin{table}
    \centering
    \begin{tabular}{l c}
    \toprule
    {\bf Parameter} & {\bf Value} \\
    \midrule   
         \# Marg. Epochs & 100 \\
         \# Cond. Epochs & 100 \\
         Marg. Lr & 0.0001 \\
         Cond. Lr & 0.001\\
         \hline
         \# FF-layers & 2 \\
         \# Marg. Modes (Marginal Gaussian Mixture) &  8 \\
         \# Cond. Modes (Conditional Gaussian Mixture) &  8 \\
         Covariance type & Diagonal \\
    \midrule
    \end{tabular}
    \caption{KNIFE hyper-parameters. Selection of hyper-parameters is based on the extensive study provided in~\cite{darrinWhenEmbeddingModel2024}}
    \label{tab:knife-hparams}
\end{table}

\section{Details about the data}
\label{app:data}

We provide below, an example of the data we used along with the different annotation schemes.

\paragraph{Input.} Voters in strife torn Colombia go to the polls today in local elections in the midst of a wave of violence directed by armed groups of the left and the right against many of the candidates. VOA 's Bill Rodgers has a report from our South American bureau. Some 23 million Colombians are registered to vote Sunday to elect governors, mayors and other local officials in the South American nation, but the election is taking place in the midst of a rising wave of assassinations, kidnappings and threats against the candidates by leftist guerillas and right wing paramilitary groups. Twenty mayoral candidates have been killed and more than 200 kidnapped in recent months. Election authorities say 0 another 200 politicians withdrew their candidacies after being threatened. At the same time, these armed groups are backing reported to have fielded a number of stealth candidates in an effort to expand its control of Colombian territory. Despite the violence, the government of President Andres Pastrana has refused to suspend Sunday's elections vowing 0 they will take place throughout the country and in a democratic atmosphere. Bill Rodgers, VOA News, South America Bureau.

\paragraph{SUM.} Colombia faces violence from armed groups targeting candidates in local elections; government vows to continue voting as planned.

\paragraph{COR.} "strife - torn Colombia" refers also to "the South American nation" and "the government of President Andres Pastrana" and "the country". "local elections" refers also to "the election". "armed groups of the left and the right" refers also to "leftist guerillas and right - wing paramilitary groups" and "these armed groups". "VOA 's" refers also to "our" and "VOA News". "Bill Rodgers" refers also to "VOA 's Bill Rodgers". "South America Bureau" refers also to "VOA 's South American bureau". "Sunday" refers also to "Sunday 's". "the violence" refers also to "a rising wave of assassinations , kidnappings and threats against the candidates by these armed groups".

\paragraph{NER.} This text contains the following entity name : "Colombia" and the following list of dates : "today, Sunday, recent months" and the following list of organisation names : "VOA, VOA News" and the following list of person names: "Bill Rodgers, Andres Pastrana" and the following list of person types: "South American, Colombians, Colombian" and the following list of numbers: "23 million, Twenty, more than 200, 200" and the following facility name: "South America Bureau".

\paragraph{SEM.} have.03("VOA's","a report"); vote.01("Some 23 million","to"); say.01("Election authorities","0"); withdraw.01("another 200","their candidacies"); field.01("these armed groups","a number"); refuse.01("the government","to"); suspend.01("the government","Sunday 's"); vow.01("the government","0").

\paragraph{SYN.} Voters; VOA 's Bill Rodgers; Some 23 million Colombians; the election; Twenty mayoral candidates; more than 200; Election authorities; another 200 politicians; these armed groups; the government; they.

\section{Additional results}
In this section, we provide additional results that were not presented as main results, but which can help on the comprehension of the presented results.

\subsection{Layer selection}\label{app:add-res:layer-selec}

To select layers and ease the estimation of IS, we provided a simulation on which we compared fine-tuned models to pre-trained models in terms of information sufficiency. IS is firstly used in this study as an inclusion metric. However, its first formulation is a lower bound of the mutual information. Thus when estimating the IS from a fine-tuned model to a pre-trained model we estimate the amount of information the model as lost or gain after the fine-tuning operation, compared to the pre-trained model. \autoref{fig:layer-selection} provides results of such comparison for all models present in this study. The first thing we can notice, is that the fine-tuning operation makes the model loose information about the pre-trained model ($\Irond_S(ft \rightarrow pt) \leq \Irond_S(pt \rightarrow ft)$), as if the fine-tuning operation consisted in applying a mask on the pre-trained model. The second observation and it is the one that is used in this study, is that layers where this gap is the biggest is between 10 and 15 suggesting that it is in-between these layers that the fine-tuned model has lost most information about the pre-trained model, suggesting that the task is mostly encoded here, which is the reason why we used these layers.

\begin{figure}
    \centering
    \begin{subfigure}{0.49\textwidth}
        \centering
        \includegraphics[width=0.7\textwidth]{figures/mi/lora/mistral_base_pretrained_vs_finetuned_layer_evolution.pdf}
        \caption{Mistral (Base)}
    \end{subfigure}
    \begin{subfigure}{0.49\textwidth}
        \centering
        \includegraphics[width=0.7\textwidth]{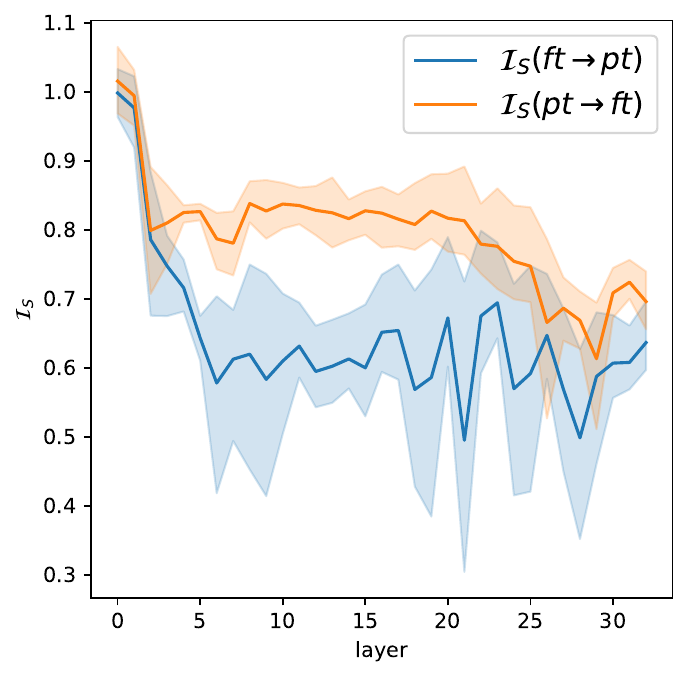}
        \caption{Mistral (Instruct)}
    \end{subfigure}
    \begin{subfigure}{0.49\textwidth}
        \centering
        \includegraphics[width=0.7\textwidth]{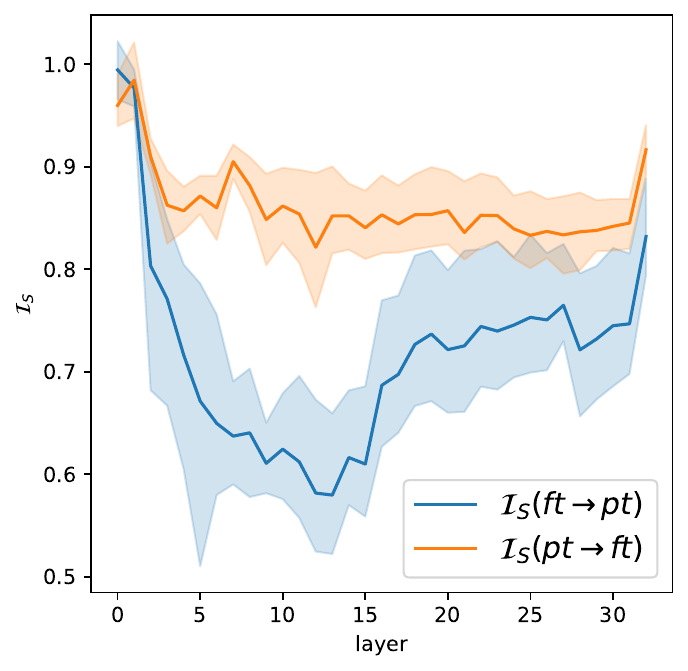}
        \caption{Llama 3 (Base)}
    \end{subfigure}
    \begin{subfigure}{0.49\textwidth}
        \centering
        \includegraphics[width=0.7\textwidth]{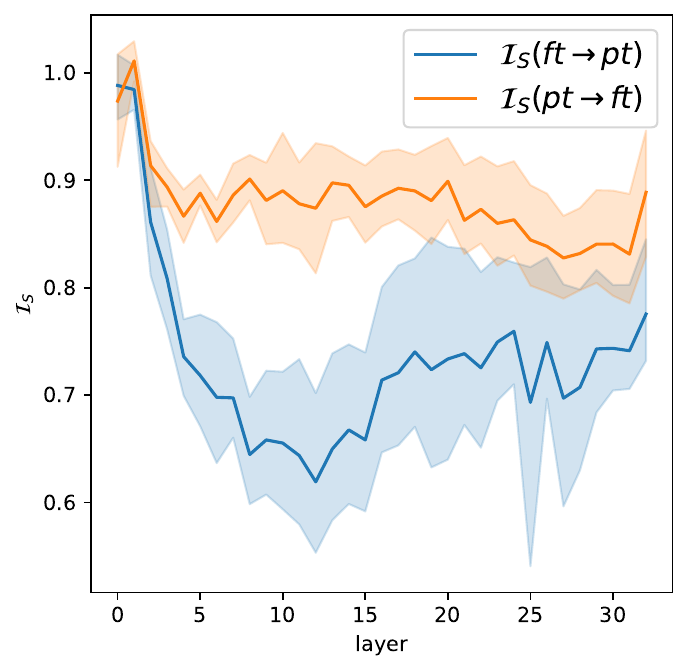}
        \caption{Llama 3 (Instruct)}
    \end{subfigure}
    \caption{Information sufficiency comparison between pre-trained models and each corresponding fine-tuned model.}
    \label{fig:layer-selection}
\end{figure}

\subsection{Information sufficiency matrices}\label{app:additional-figures}

Additionally to~\autoref{fig:mean-mi}, we provide on\autoref{fig:info-suff-10-15} information sufficiency matrices for all the models. As a complement of~\autoref{tab:ranking}, we provide on~\autoref{fig:pp-detail} details of the PP ranking of the different tasks and for the different models. A positive PP for a task means that the task contains the others while the others do not contain this task. We can see that the only tasks with positive PP are summarization and COREF, pointing once again the difficulty of these two tasks and the need for these tasks of various linguistic skills to be performed correctly, which is once again in line with premises about the NLP pipeline.

\begin{figure}
    \centering
    \begin{subfigure}{0.45\textwidth}
        \centering
        \includegraphics[width=\textwidth]{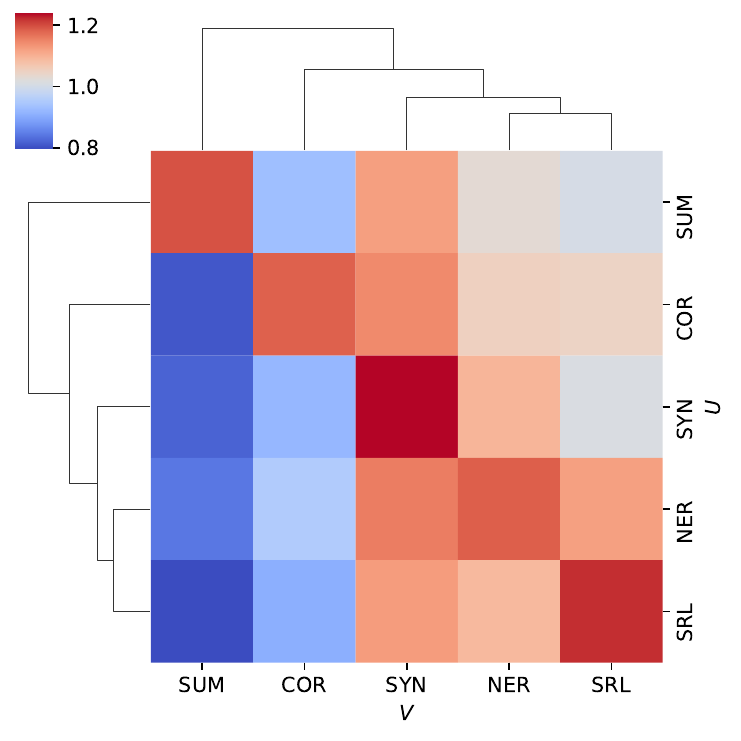}
        \caption{Mistral Base}
    \end{subfigure}
    \begin{subfigure}{0.45\textwidth}
        \centering
        \includegraphics[width=\textwidth]{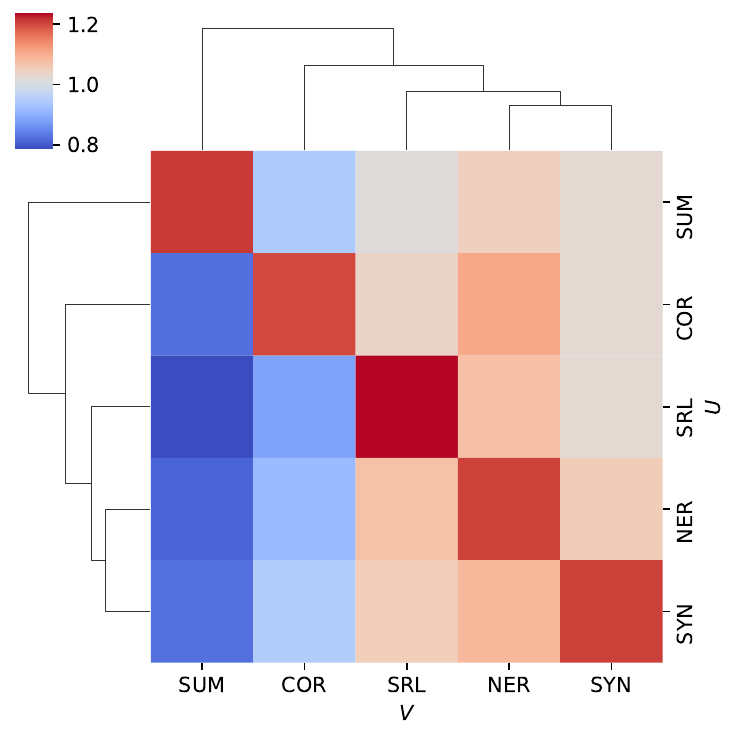}
        \caption{Mistral Instruct}
    \end{subfigure}
    \begin{subfigure}{0.45\textwidth}
        \centering
        \includegraphics[width=\textwidth]{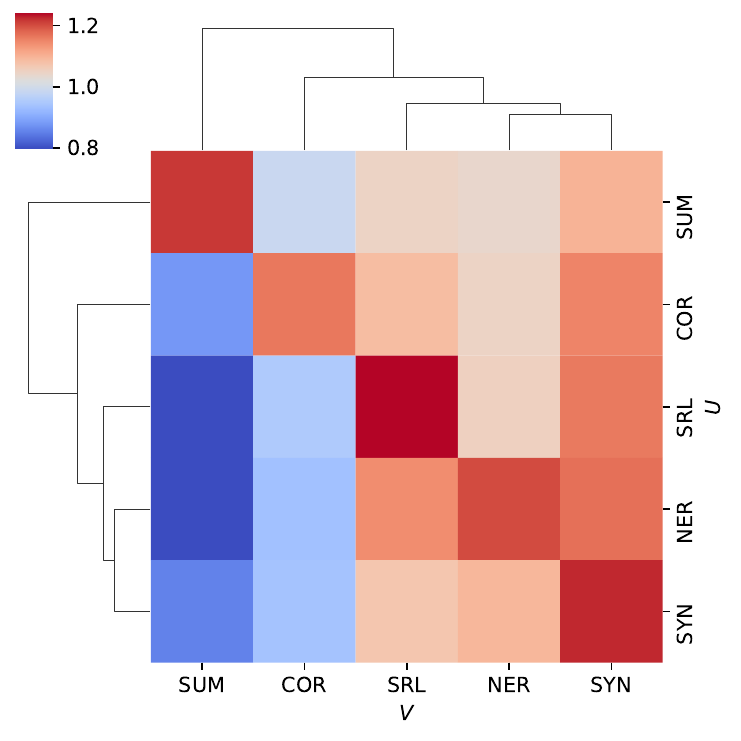}
        \caption{Llama 3 Base}
    \end{subfigure}
    \begin{subfigure}{0.45\textwidth}
        \centering
        \includegraphics[width=\textwidth]{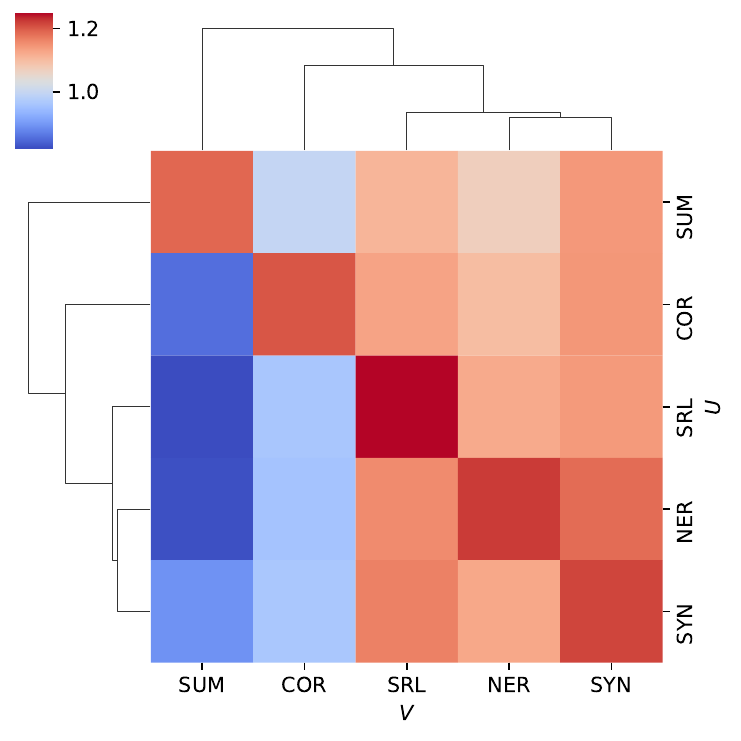}
        \caption{Llama 3 Instruct}
    \end{subfigure}
    \caption{Information sufficiency averaged between layers 10 and 15.}
    \label{fig:info-suff-10-15}
\end{figure}

\begin{figure}
    \centering
    \begin{subfigure}{0.45\textwidth}
        \centering
        \includegraphics[width=.3\textwidth]{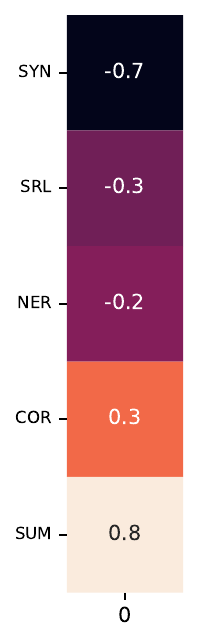}
        \caption{Mistral Base}
    \end{subfigure}
    \begin{subfigure}{0.45\textwidth}
        \centering
        \includegraphics[width=.3\textwidth]{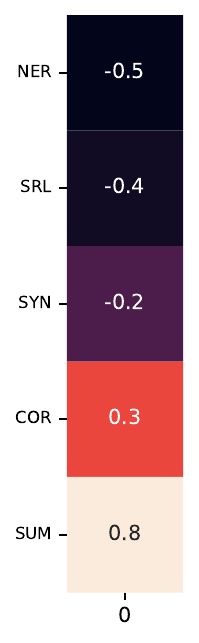}
        \caption{Mistral Instruct}
    \end{subfigure}
    \begin{subfigure}{0.45\textwidth}
        \centering
        \includegraphics[width=.3\textwidth]{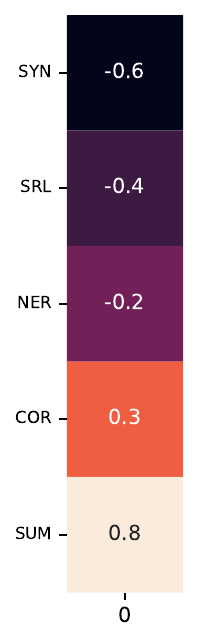}
        \caption{Llama 3 Base}
    \end{subfigure}
    \begin{subfigure}{0.45\textwidth}
        \centering
        \includegraphics[width=.3\textwidth]{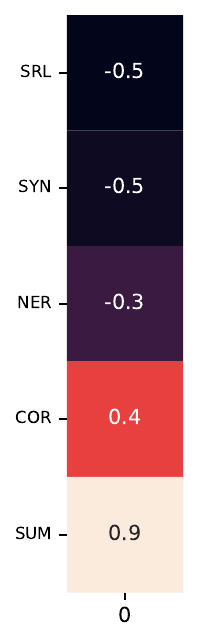}
        \caption{Llama 3 Instruct}
    \end{subfigure}
    \caption{Information sufficiency averaged between layers 10 and 15.}
    \label{fig:pp-detail}
\end{figure}

\subsection{Correlation with the naive approach}

One of the most intuitive way to measure tasks' inclusion is to train a model on a task $U$ and evaluate its performances on a task $V$. Considering we are using generative language models in this study, this can easily be done by a simple change of prompt. Thus we can measure the performance of a model trained on $U$ on the task $V$, using classical metric such as BERTScore~\cite{bert-score} or ROUGE~\cite{lin-2004-rouge}. This new evaluation of the inclusion gives new matrices such as the one presented on~\autoref{fig:info-suff-10-15}, except that these new matrices present the metric performances of a cross task evaluation. We present on~\autoref{tab:cross-task-vs-mi} correlations in terms of Kendall-$\tau$~\cite{kendall1938new} coefficient between IS matrices and cross task performances for ROUGE and BERTScore. Low correlation results suggest that there is no clear correlations between these two approaches. The main explanation being the alignment problem between tasks, considering that the output format for each task is different.

\begin{table}
\centering
\begin{tabular}{lcc}
\toprule
     & {\bf BERT} & {\bf ROUGE} \\
\midrule
Llama 3 base & 0.02 & 0.43 \\
Llama 3 instruct & -0.03 & 0.37 \\
\hline
Mistral base & 0.23 & 0.43 \\
Mistral instruct & 0.05 & 0.29 \\
\bottomrule
\end{tabular}
\caption{Kendall-$\tau$ between information sufficiency and naive cross evaluation set-up.}\label{tab:cross-task-vs-mi}
\end{table}

\section{Layer ablation study}
\label{app:layer-ablation}

In our experiments, we proposed a method to select the most relevant layers on which to calculate IS. We propose here an ablation study on the layer selection. We propose here several variations of~\autoref{tab:ranking}, based on different layer selections, to understand which layer is interesting to discover the NLP pipeline. Our first analysis, is based on layers from 10 to 15 based on~\autoref{fig:info-suff-10-15}. Our main argument being that deepest layers (after 15) are not relevant to understand the behavior of the task with respect to the model. First, on~\autoref{tab:ranking-all-layers} we propose the same analysis, based on all the layers. We can observe that in that case, the NLP pipeline is not fully respected with a change of position of NER and SRL. In order to better understand the layers that can actually interfere with pipeline discovery, we are carrying out two new rankings. The first is based on layers 10 to 33 (preservation of deep layers) and is reported on~\autoref{tab:ranking} on which we can once again see the perturbation between SRL and NER. Then we perform the same ranking with layers between 1 and 20 (focus on lower layers) which we report on~\autoref{tab:ranking-1-20}. On this last ranking we can see that the same pipeline as the one presented in the main study is respected.
We can see that deepest layers can effectively introduce noise to discover relationship between tasks, which is once again in line with known works. In order to better visualize this behavior, we report on~\autoref{fig:is-layer-evol}, the evolution of the values of the IS for our different task combinations. We observe interesting behaviors mainly between layers 10 and 15, with a spike in IS values, suggesting that IS finds relationships between tasks at this level, while in the deeper layers, values are lower, which can once again be explained by differences in output formats. We can observe high values of IS in the very early layers, which can largely be explained that lower layers will mostly remain unchanged during fine-tuning due to gradient vanishing problems.

Generally speaking, we can see that changes in the layers only affect the position of two tasks (NER and SRL) while the other ones remain essentially untouched. What's more, we can see from this analysis that in these different set-ups, the Base models deliver the same information overall, while the instructed models differ in the given rankings, as highlighted in the body of the study.

\begin{table}[t!]
    \centering
    {\small 
        \begin{tabular}{l|c|c@{~}cc@{~}c}
            \toprule
                  &       & \multicolumn{2}{c}{\bf Llama 3} & \multicolumn{2}{c}{\bf Mistral} \\
                  & {\bf Avg.} & {\bf Base} & {\bf Instruct} & {\bf Base} & {\bf Instruct} \\
            \midrule
            SUM & 4.0 & 4 & 4 & 4 & 4 \\
            COR & 3.0 & 3 & 3 & 3 & 3 \\
            {\bf SRL} & 1.25 & 2 & 0 & 2 & 1 \\
            {\bf NER} & 1.0 & 1 & 2 & 1 & 0 \\
            SYN & 0.75 & 0 & 1 & 0 & 2 \\
            \bottomrule
        \end{tabular}
    }
    \caption{Ranking on all layers}
    \label{tab:ranking-all-layers}
\end{table}

\begin{table}[t!]
    \centering
    {\small 
    \begin{tabular}{l|c|c@{~}cc@{~}c}
        \toprule
              &       & \multicolumn{2}{c}{\bf Llama 3} & \multicolumn{2}{c}{\bf Mistral} \\
              & {\bf Avg.} & {\bf Base} & {\bf Instruct} & {\bf Base} & {\bf Instruct} \\
        \midrule
        SUM & 4.0 & 4 & 4 & 4 & 4 \\
        COR & 3.0 & 3 & 3 & 3 & 3 \\
        {\bf SRL} & 1.25 & 2 & 0 & 2 & 1 \\
        {\bf NER} & 1.0 & 1 & 2 & 1 & 0 \\
        SYN & 0.75 & 0 & 1 & 0 & 2 \\
        \bottomrule
    \end{tabular}
    }
    \caption{Ranking on layer 10 to 33}
    \label{tab:ranking-10-33}
\end{table}

\begin{table}[t!]
    \centering
    {\small 
        \begin{tabular}{l|c|c@{~}cc@{~}c}
            \toprule
                  &       & \multicolumn{2}{c}{\bf Llama 3} & \multicolumn{2}{c}{\bf Mistral} \\
                  & {\bf Avg.} & {\bf Base} & {\bf Instruct} & {\bf Base} & {\bf Instruct} \\
            \midrule
            SUM & 4.0 & 4 & 4 & 4 & 4 \\
            COR & 3.0 & 3 & 3 & 3 & 3 \\
            {\bf NER} & 1.75 & 2 & 2 & 2 & 1 \\
            {\bf SRL} & 0.75 & 1 & 1 & 1 & 0 \\
            SYN & 0.5 & 0 & 0 & 0 & 2 \\
            \bottomrule
        \end{tabular}
    }
    \caption{Ranking on layer 1 to 20}
    \label{tab:ranking-1-20}
\end{table}

\begin{figure}
    \centering
    \begin{subfigure}{0.45\textwidth}
        \centering
        \includegraphics[width=\textwidth]{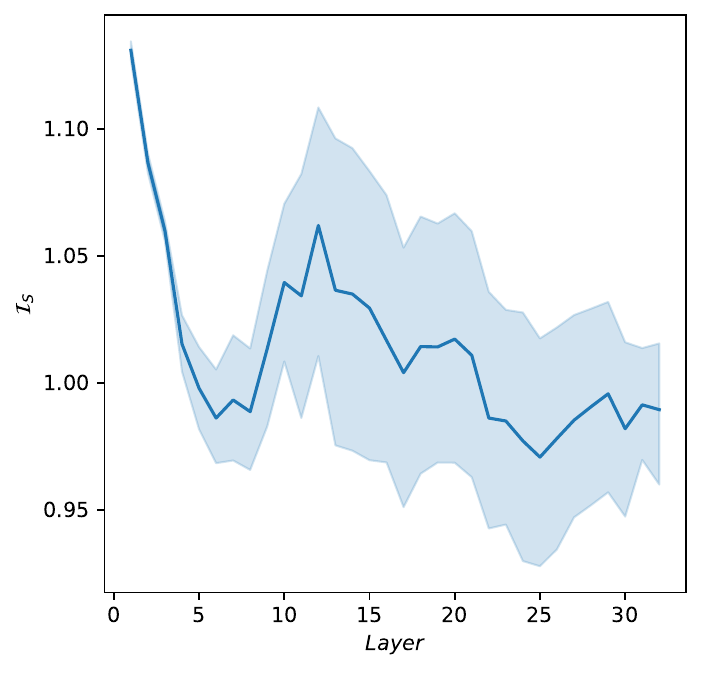}
        \caption{Mistral Base}
    \end{subfigure}
    \begin{subfigure}{0.45\textwidth}
        \centering
        \includegraphics[width=\textwidth]{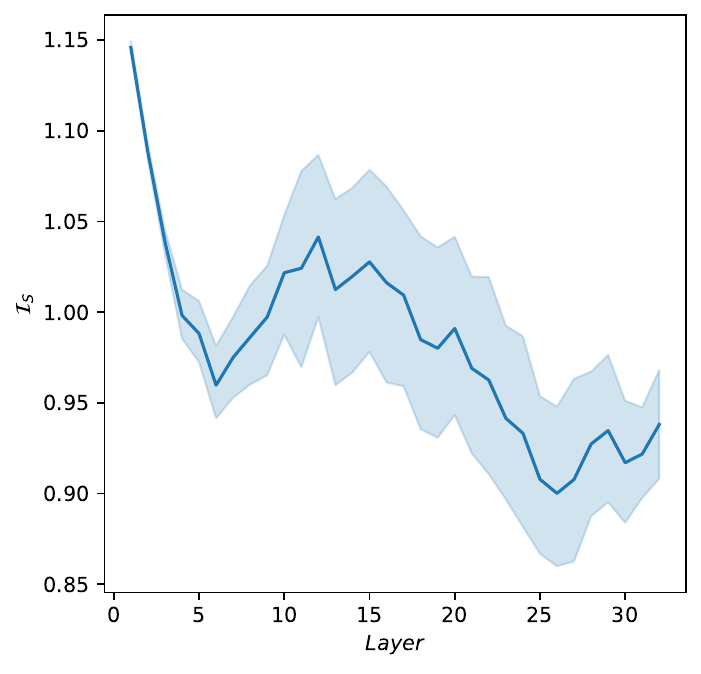}
        \caption{Mistral Instruct}
    \end{subfigure}
    \begin{subfigure}{0.45\textwidth}
        \centering
        \includegraphics[width=\textwidth]{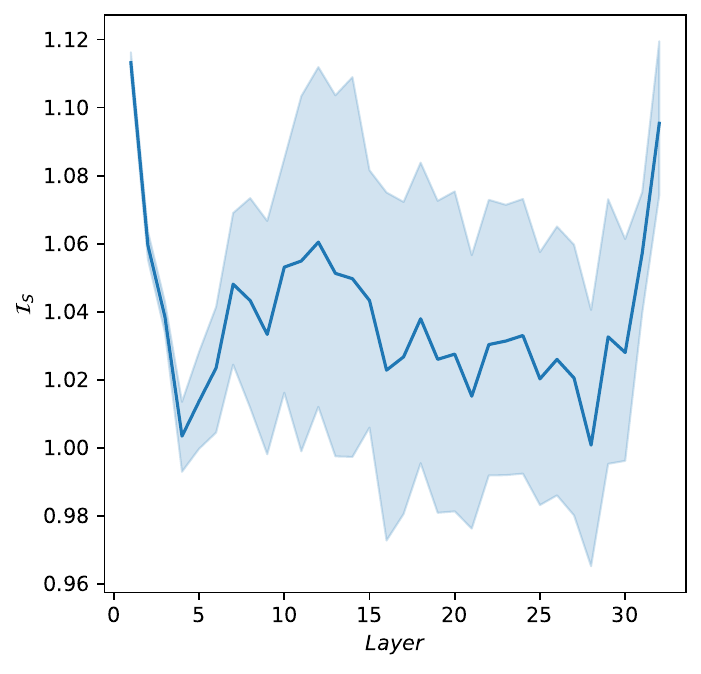}
        \caption{Llama Base}
    \end{subfigure}
    \begin{subfigure}{0.45\textwidth}
        \centering
        \includegraphics[width=\textwidth]{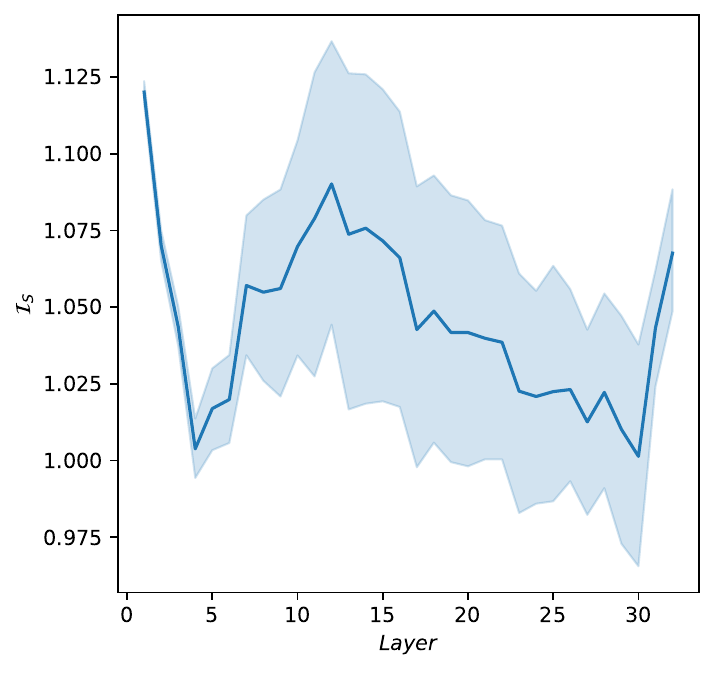}
        \caption{Llama Instruct}
    \end{subfigure}
    \caption{Information sufficiency evolution across layers.}
    \label{fig:is-layer-evol}
\end{figure}

\begin{figure}
    \centering
    \begin{subfigure}{0.45\textwidth}
        \centering
        \includegraphics[width=\linewidth]{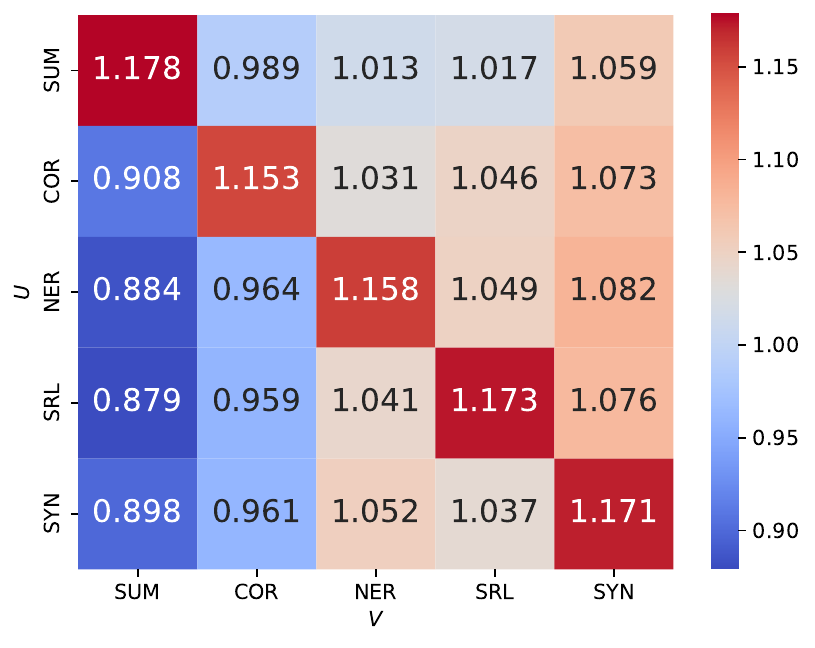}
        \caption{Base Only \label{fig:subfig:base}}
    \end{subfigure}
    \begin{subfigure}{0.45\textwidth}
        \centering
        \includegraphics[width=\linewidth]{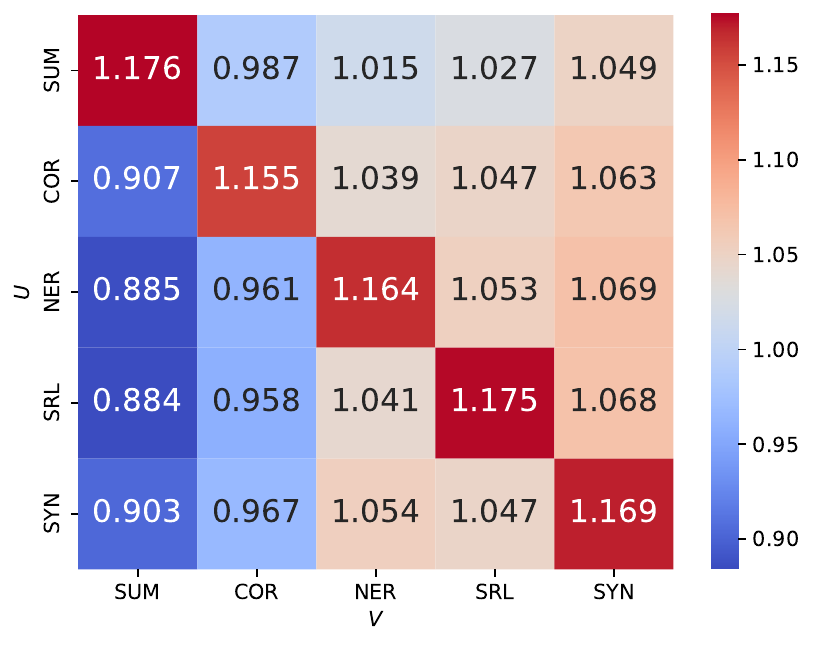}
        \caption{Instruct models \label{fig:subfig:instruct}}
    \end{subfigure}
    \caption{IS heat, for layers ranging from 1 to 20 with an average over only Base Models (\autoref{fig:subfig:base}), or on only Instruct models (\autoref{fig:subfig:instruct}).}
\end{figure}

\section{Task vector approach}
\label{app:task-vectors}

Task vectors~\cite{ilharcoEditingModelsTask2023} are objects of growing interest in the community. They were firstly defined in the case of transfer learning, where a pre-trained model is fine-tuned on different downstream tasks. In that case, the task vector was defined as follows:
\begin{defn}
    Let $W_0 \in \mathbb{R}^{m\times n}$, be the weights of a pre-trained model\footnote{The weights of a model can always be represented as a matrix or a vector. For some connection with current results, we chose the matrix representation.}, and let $W_U$ be the weights of the same model, but after fine-tuning it on a task $U\equiv \Prob_{XY_U}$. The task vector of task $U$ is given by:
    \begin{equation}
        \tau_U = W_U - W_0.
    \end{equation}
   \label{def:task-vector}
\end{defn}
These objects were used to combine properties of a model through arithmetic operations~\cite{ortiz-jimenezTaskArithmeticTangent, jin-arxiv-24, zhang-neurips-23}, or to remove certain components such as toxicity, personal information or bias from a language model~\cite{gaoEthosRectifyingLanguage2024,zhang-neurips-23,liu-arxiv-24}. In this study we used Low Rank Adaptation (\lora~\cite{lora}) for our fine-tunings on the different tasks. In this special case we have, 
\[
    W_U = W_0 + B_UA_U \quad \text{with} \quad B_U\in\mathbb{R}^{m\times r},~A_U\in\mathbb{R}^{r \times n},
\]
where $r \in \mathbb{N}$ is the chosen rank (which is 8 in this study). Given this definition of \lora, we have,
\[
    \tau_U = B_UA_U.
\]
In this special case, comparing task vectors, is equivalent as comparing different product $B_UA_U$. We propose here several distances in order to compare task vectors.

\paragraph{Cosine distance.} One of the main distances to compare semantically two vectorial objects is the cosine similarity. We propose here to simply define this distance as following,
\[
    \dist_{\cos}\left(\tau_U, \tau_V\right) \triangleq 1-\cos\left(\textrm{flatten}(B_UA_U), \textrm{flatten}(B_VA_V)\right).
\]

\paragraph{$L_2$ distance.} A standard way to compare vectorial representations is through euclidean distances, which we define as following, 
\[
    \dist_{L_2}(\tau_U, \tau_V) \triangleq \| \textrm{flatten}(B_UA_U) - \textrm{flatten}(B_VA_V) \|_2.
\]
Due to the dimension of task vectors euclidean distances are really restrictive.

\paragraph{Grassmann distance.} Another way to compare task vectors is to see them as vector spaces. In fact when using \lora, task vectors are defined through matrices which define vector spaces (column space). Grassmann distance is a mathematically well defined distance between vector spaces. It is based on the notion of principal angles between vector spaces~\cite{afriat-mpcps-57,miao-algebra-92}. If $W_U$ and $W_U$ are two matrices of rank $r$ whose columns are orthonormal\footnote{if they are not, a simple singular value decomposition of the matrices can allow us to do so}, then we can consider $\sigma$, the set of eigenvalues of $W_U^TW_V$. A theoretical result given by~\citet{bjorck-mc-73} is that these eigenvalues are the cosines of the principal angles between the image spaces of $W_1$ and $W_2$ {\i.e.} if we set $\theta$ to be the set of principal angles between $\imSpace(W_U)$ and $\imSpace(W_V)$, then $\cos(\theta)=\sigma$. Based on this information, we then have~: 
\begin{equation}
  \dist_G(W_U, W_V) = \sqrt{\sum_{i=1}^r\left(\theta_i\right)^2} \leq \sqrt{r}\frac{\pi}{2}.
  \label{eq:grassmann:distance}
\end{equation}
An additional remark on this Grassmann distance is that if we consider $(W_U, W_V) \in \mathbb{R}^{d\times d}$, two matrices of maximum rank $d$, then $\imSpace(W_U) = \imSpace(W_V) = \mathbb{R}^d$ implying that $\dist_G(W_1, W_2)=0$. This distance is therefore only of interest for matrices that are not of maximal rank, making it particularly interesting in the context of \lora. Based on this, in our context, the Grassmann distance between task vectors will only be,
\[
    \dist_G(\tau_U, \tau_V) \triangleq \dist_G(B_UA_U, B_VA_V).
\]
Grassmann distance has already been used in~\cite[Appendix G.]{lora} in the context of measuring the covering between different \lora\ adaptations. In our context, the covering is between different task vectors, and thus by extrapolation, between different tasks. An interesting property about Grassmann distance, is the following,
\begin{prop}[Grassmann distance is $A$-invariant.]\label{prop:grassmann-distance-A-invariance}
    In the case of {\lora} of rank $r$, if $\rk(A_U) = \rk(A_V) = r$, then we have,
    \[
        \dist_G(B_UA_U, B_VA_V) = \dist_G(B_U,B_V).
    \]
\end{prop}
\begin{proof}
    The proof is direct by using the rank Theorem.
\end{proof}
\begin{rem}
    \autoref{prop:grassmann-distance-A-invariance} requires the following condition,
    \[
        \rk(A_U) = \rk(A_V) = r.
    \]
    However, as stated in~\cite{malekmohammadi2024implicit}, when using {\lora}, $A$ matrices remain essentially unchanged, and because they are initialized randomly by using a Gaussian distributions, the probability of having full rank $A$ matrices is $1$. Moreover in our experiments, we checked this property empirically, allowing thus to use~\autoref{prop:grassmann-distance-A-invariance}.
\end{rem}

\begin{rem}
    Every results presented here is to be interpreted as distances (the higher the less similar).
\end{rem}

\subsection{Link with Information sufficiency.}

Language models we are using are essentially continuous and differentiable applications with respect to its parameters. Thus a small change in the parameters of the models will induce a small change in the outputs of the models (this is direct application of the definition of continuity). Implying thus that closeness in the task vectors will induce little variation in the activation space, the inverse being not necessary true. This implies that a small distance between task vectors, will most likely provide high information sufficiency. However a high information sufficiency can be discovered between task vectors that are far from each other.

\subsection{Result analysis}

In the context of our models, we decided to apply \lora\ on query and value projections on the different layers of the models. Thus for every tasks we defined and for every models, we have,
\[  
    \begin{split}
    & \tau^{Q}_t \triangleq (B^{Q,l}_t, A^{Q,l}_t) \quad \text{For the Query projection bloc at layer $l$}\\
    & \tau_t^{V} \triangleq (B^{V,l}_t, A^{V,l}_t) \quad \text{For the Value projection bloc at layer $l$}
    \end{split}
\]
Since in the literature, it is a well known fact that Query and Value projection encode different information, we decided to separate the analysis between Queries and Values, by looking at the average distance across layer for each module, {\it i.e.}
\begin{equation}
    \begin{split}
        & \dist\left(\tau^{\queryProj}_U, \tau^\queryProj_V\right) \triangleq \frac{1}{L} \sum_{l=1}^L \dist\left( B^{\queryProj,l}_UA^{\queryProj,l}_U, B^{\queryProj,l}_VA^{\queryProj,l}_V  \right)\quad \text{Query distance}\\
        & \dist\left(\tau^{\valueProj}_U, \tau^\valueProj_V\right) \triangleq \frac{1}{L} \sum_{l=1}^L \dist\left( B^{\valueProj,l}_UA^{\valueProj,l}_U, B^{\valueProj,l}_VA^{\valueProj,l}_V  \right)\quad \text{Value distance}
    \end{split}
\end{equation}
\autoref{fig:task-vector:heat-map:mistral-base} and~\autoref{fig:task-vector:heat-map:mistral-instruct} provide distance results for the Mistral model (respectively Base and Instruct). \autoref{fig:task-vector:heat-map:llama-base} and~\autoref{fig:task-vector:heat-map:llama-instruct} provide same results for the Llama 3 model. 
First we can see that results seem similar between task vectors on the Values and the Queries.
Then, in terms of Grassmann and Cosine distances, we have a closeness between SYN and SRL, which is in line with our initial results on IS. In the same way as for IS, we observe the specific character of the summarization task, with a great distance from the other tasks present.
The $L_2$ distance does not give interesting association between tasks (with respect to the underlying pipeline hypothesis), for the different studied models. This can easily be explained by the restrictive nature of the $L_2$ distance in such a high-dimensional space.

\subsection{Limitations}

One of the main limitation of this approach is its symmetric character. When assessing distances between task vectors, we rely on a symmetric approach which is not leading to interpretation of some partial ordering between tasks, which is our main goal in this study.

\begin{figure}
    \centering
    \begin{subfigure}{1\textwidth}
        \centering
        \includegraphics[width=\textwidth]{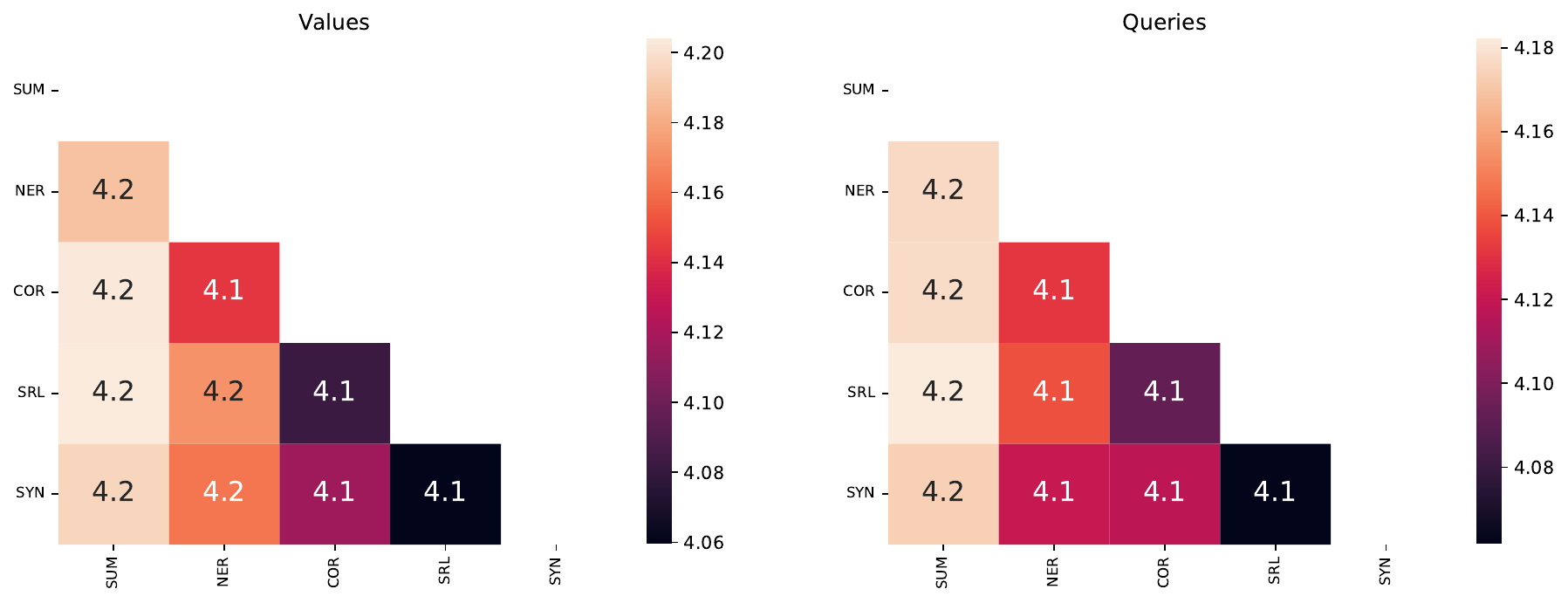}
        \caption{Grassmann Mistral Base}
    \end{subfigure}
    \begin{subfigure}{1\textwidth}
        \centering
        \includegraphics[width=\textwidth]{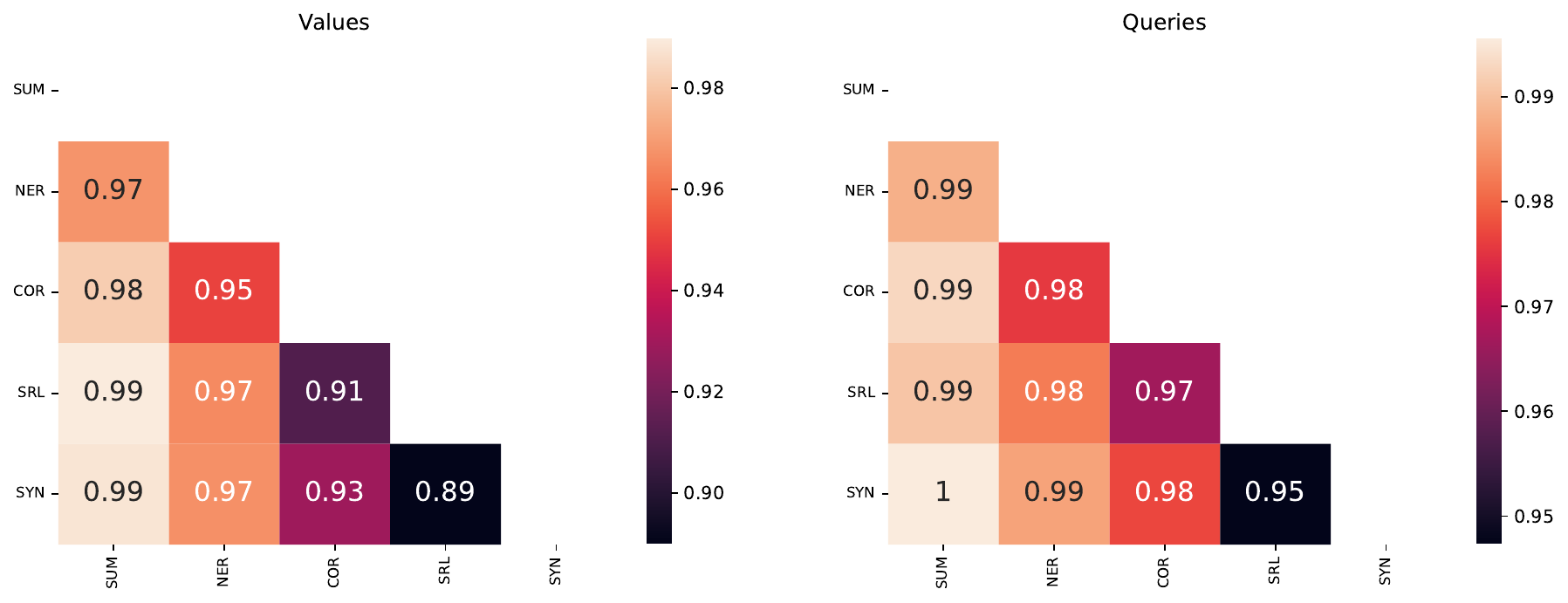}
        \caption{Cosine Mistral Base}
    \end{subfigure}
    \begin{subfigure}{1\textwidth}
        \centering
        \includegraphics[width=\textwidth]{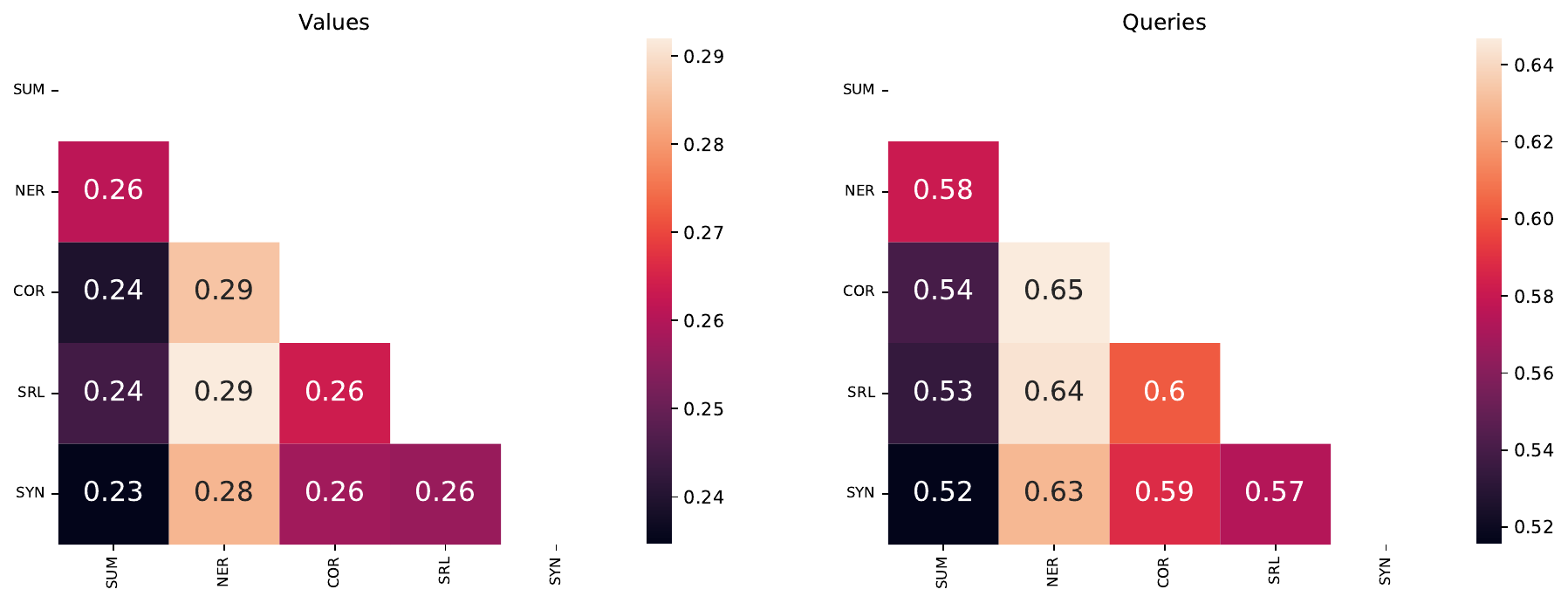}
        \caption{$L_2$ Mistral Base}
    \end{subfigure}
    \caption{Mistral Base distance heat maps}
    \label{fig:task-vector:heat-map:mistral-base}
\end{figure}

\begin{figure}
    \centering
    \begin{subfigure}{1\textwidth}
        \centering
        \includegraphics[width=\textwidth]{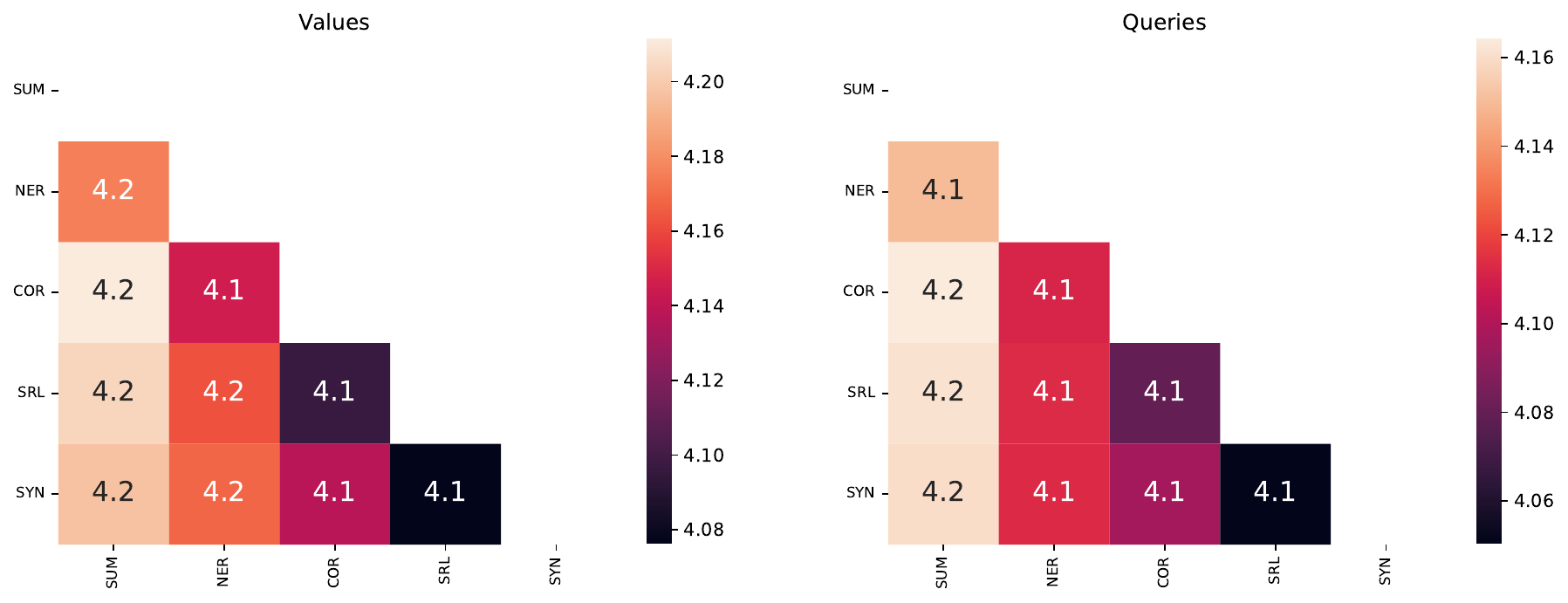}
        \caption{Grassmann Mistral Instruct}
    \end{subfigure}
    \begin{subfigure}{1\textwidth}
        \centering
        \includegraphics[width=\textwidth]{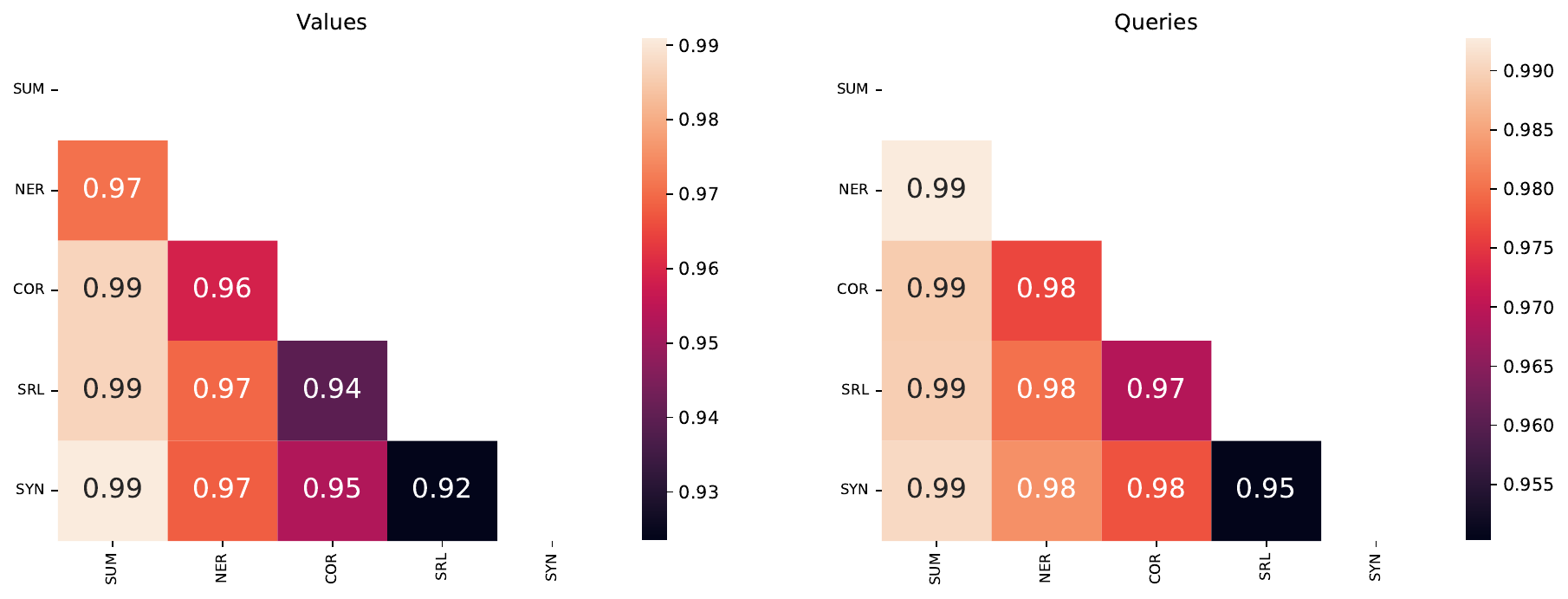}
        \caption{Cosine Mistral Instruct}
    \end{subfigure}
    \begin{subfigure}{1\textwidth}
        \centering
        \includegraphics[width=\textwidth]{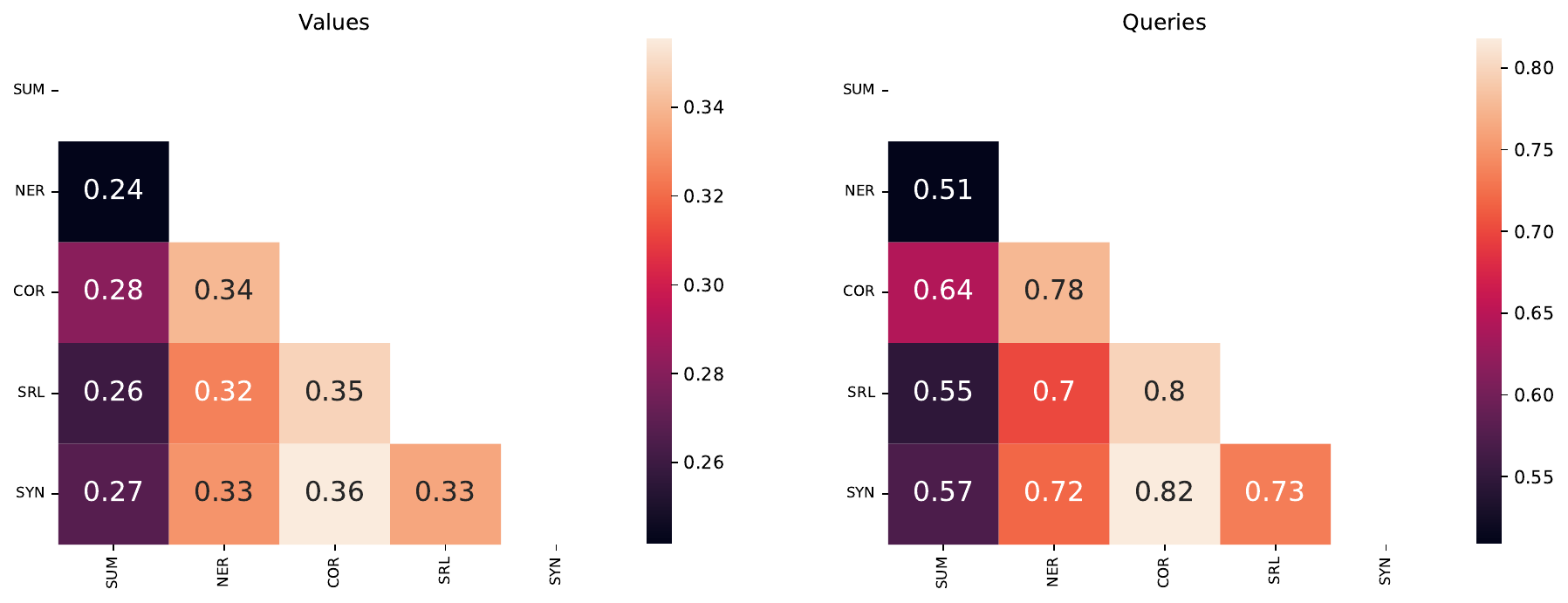}
        \caption{$L_2$ Mistral Instruct}
    \end{subfigure}
    \caption{Mistral Instruct distance heat maps}
    \label{fig:task-vector:heat-map:mistral-instruct}
\end{figure}


\begin{figure}
    \centering
    \begin{subfigure}{1\textwidth}
        \centering
        \includegraphics[width=\textwidth]{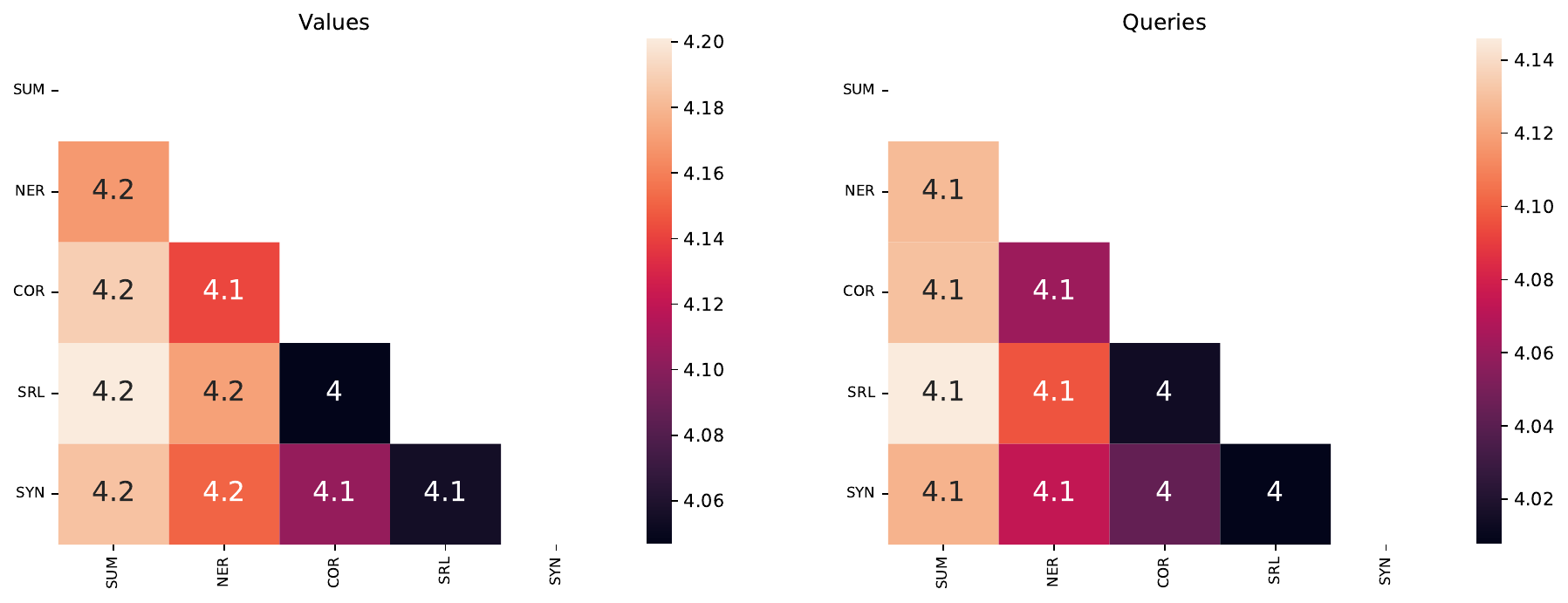}
        \caption{Grassmann Llama 3 Base}
    \end{subfigure}
    \begin{subfigure}{1\textwidth}
        \centering
        \includegraphics[width=\textwidth]{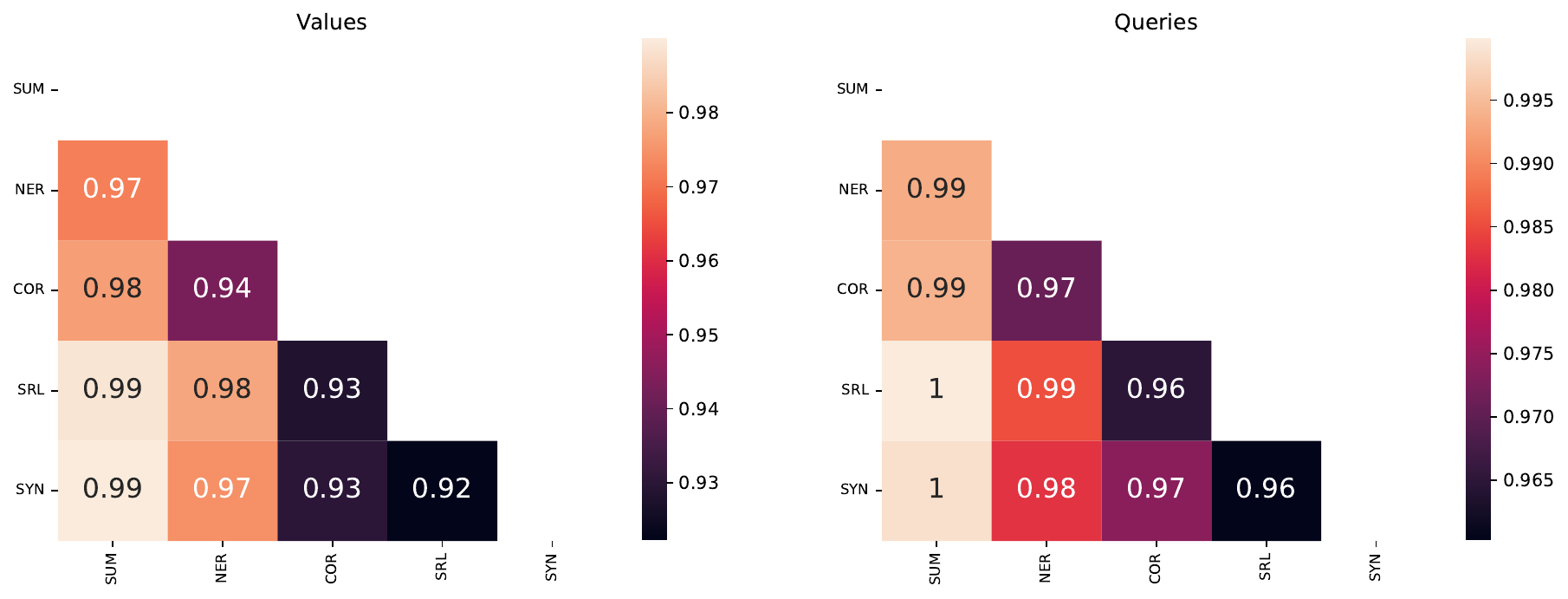}
        \caption{Cosine Llama 3 Base}
    \end{subfigure}
    \begin{subfigure}{1\textwidth}
        \centering
        \includegraphics[width=\textwidth]{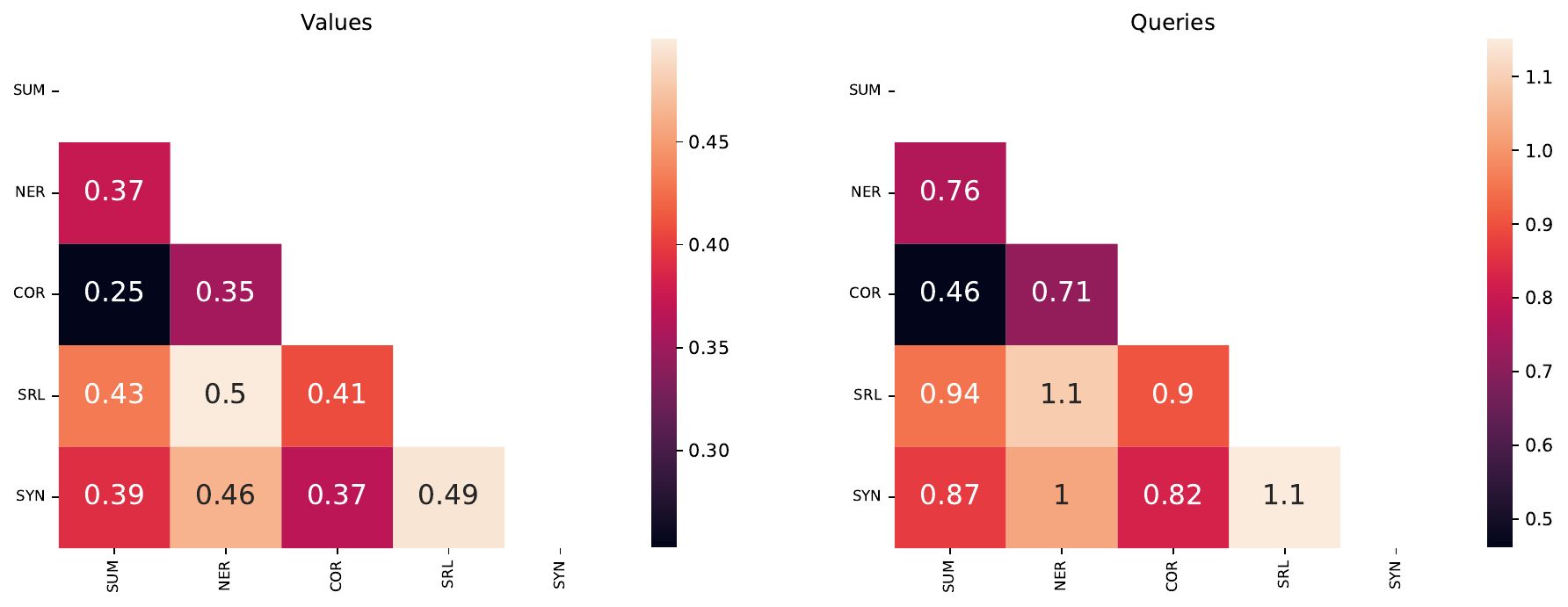}
        \caption{$L_2$ Llama 3 Base}
    \end{subfigure}
    \caption{Llama 3 Base distance heat maps}
    \label{fig:task-vector:heat-map:llama-base}
\end{figure}

\begin{figure}
    \centering
    \begin{subfigure}{1\textwidth}
        \centering
        \includegraphics[width=\textwidth]{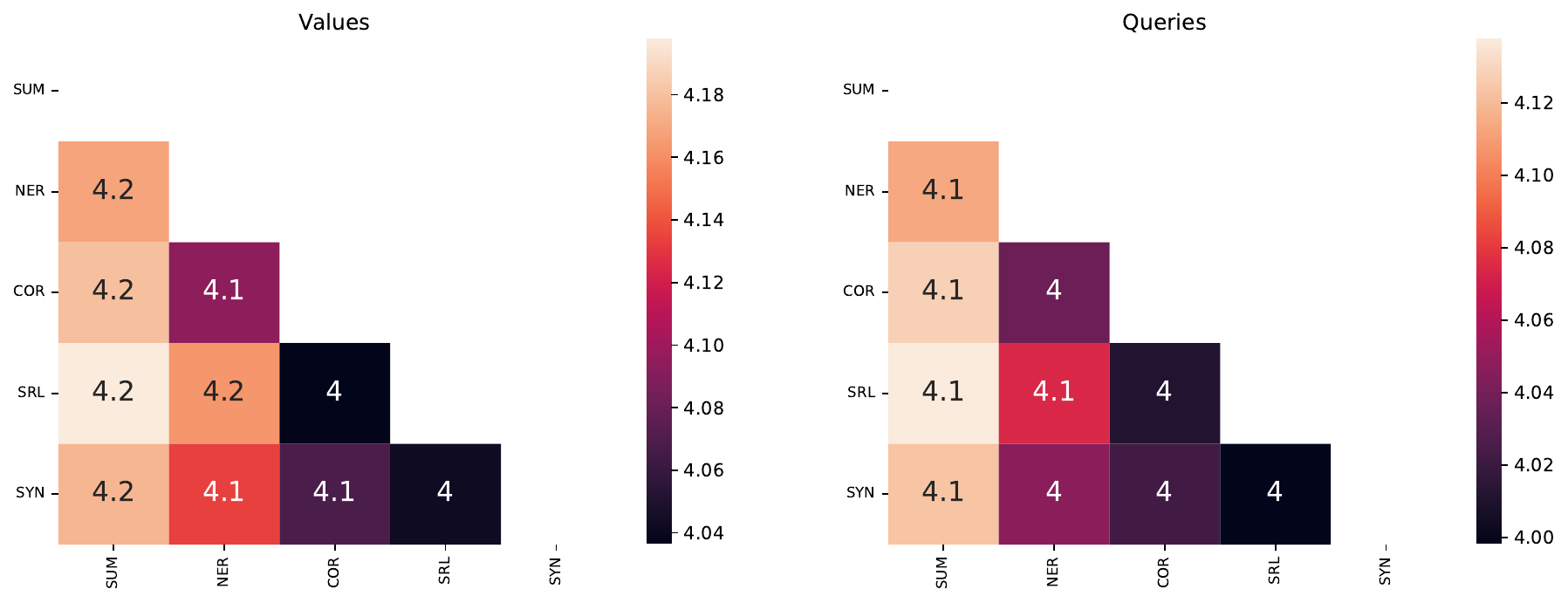}
        \caption{Grassmann Llama 3 Instruct}
    \end{subfigure}
    \begin{subfigure}{1\textwidth}
        \centering
        \includegraphics[width=\textwidth]{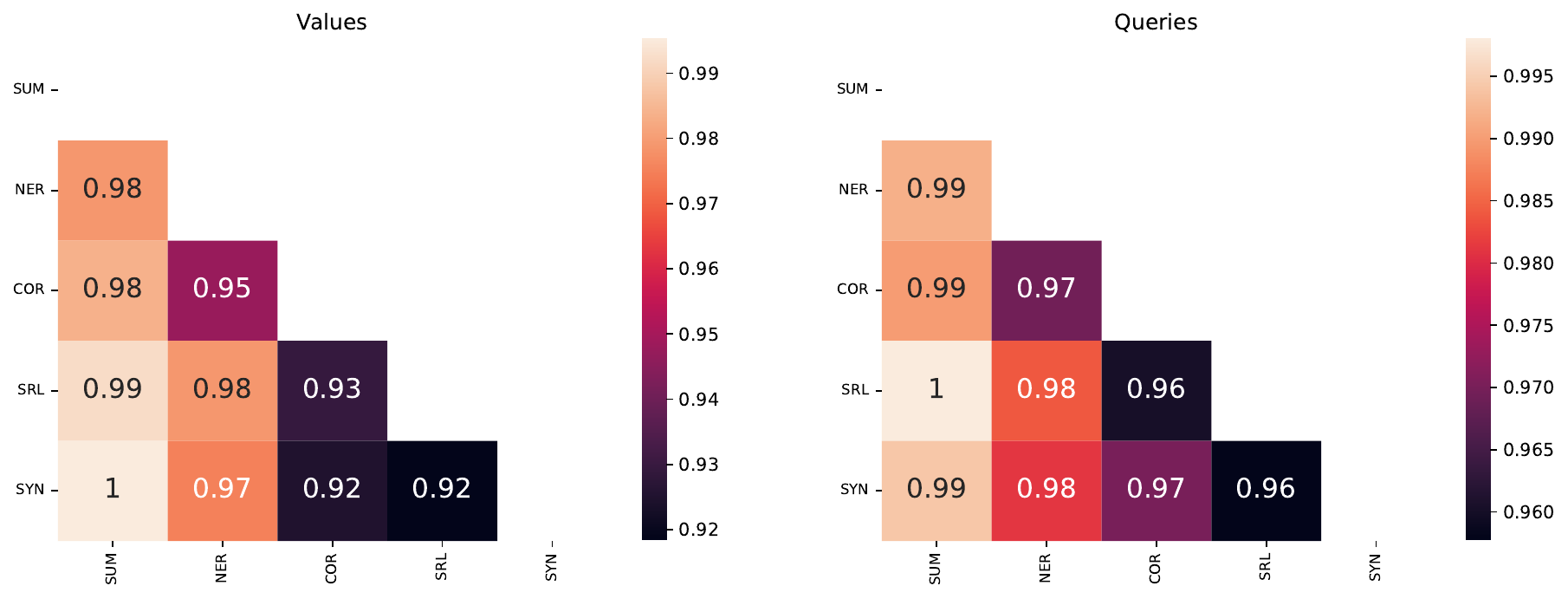}
        \caption{Cosine Llama 3 Instruct}
    \end{subfigure}
    \begin{subfigure}{1\textwidth}
        \centering
        \includegraphics[width=\textwidth]{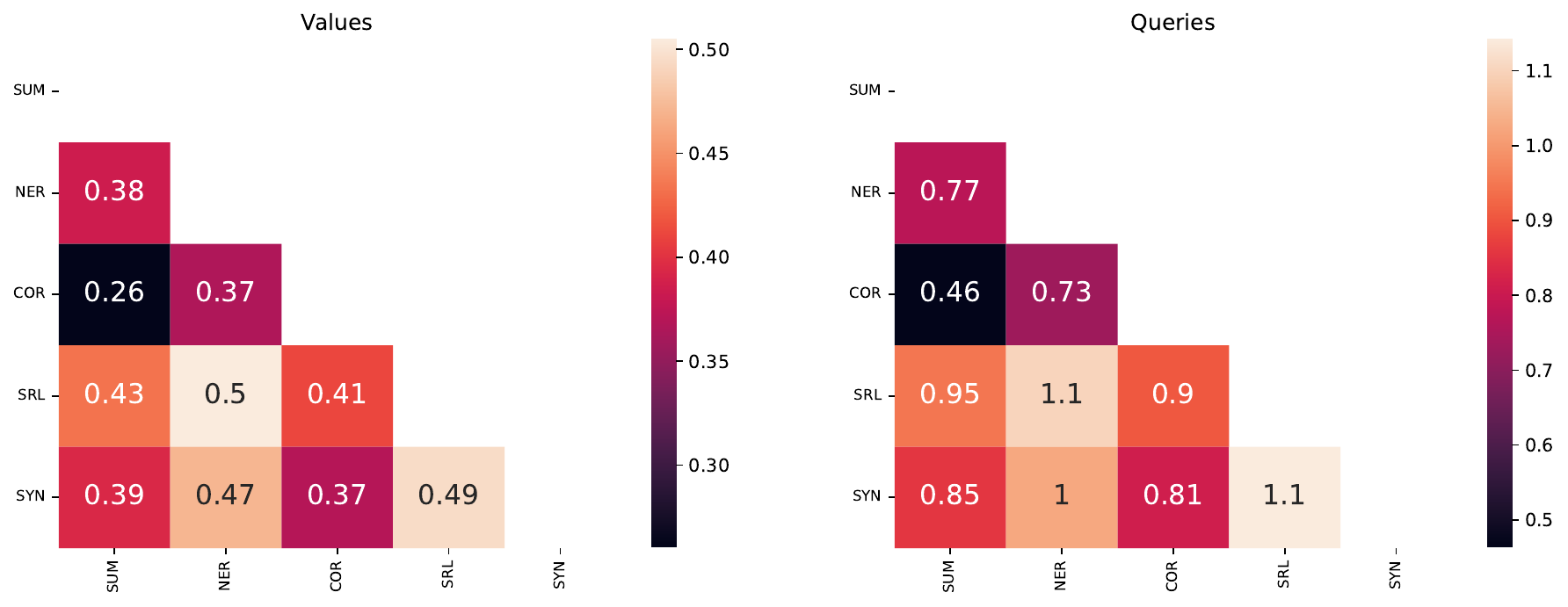}
        \caption{$L_2$ Llama 3 Instruct}
    \end{subfigure}
    \caption{Llama 3 Instruct distance heat maps}
    \label{fig:task-vector:heat-map:llama-instruct}
\end{figure}

\end{document}